\documentclass[11pt]{article}
\usepackage{fullpage,url,natbib}
\usepackage{amsthm,amsfonts,amsmath,amssymb,epsfig,color,float,graphicx,verbatim}
\usepackage{enumitem}
\usepackage{hyperref}
\hypersetup{
	colorlinks   = true, 
	urlcolor     = blue, 
	linkcolor    = blue, 
	citecolor   = blue 
}

\newtheorem{theorem}{Theorem}
\newtheorem{proposition}{Proposition}
\newtheorem*{proposition*}{Proposition}
\newtheorem{lemma}{Lemma}
\newtheorem{corollary}{Corollary}
\newtheorem{definition}{Definition}
\newtheorem{remark}{Remark}
\newtheorem{assumption}{Assumption}
\newtheorem{claim}[theorem]{Claim}

\newcommand{\secref}[1]{Section~\ref{#1}}

\renewcommand{\eqref}[1]{Eq.~(\ref{#1})}
\newcommand{\lemref}[1]{Lemma~\ref{#1}}
\newcommand{\thmref}[1]{Theorem~\ref{#1}}
\newcommand{\propref}[1]{Proposition~\ref{#1}}
\newcommand{\appref}[1]{Appendix~\ref{#1}}

\newcommand{\assref}[1]{Assumption~\ref{#1}}
\newcommand{\defref}[1]{Definition~\ref{#1}}
\newcommand{\corref}[1]{Corollary~\ref{#1}}

\newcommand{\reals}{\mathbb{R}}
\newcommand{\R}{\mathbb{R}}

\newcommand{\norm}[1]{\Vert #1 \Vert}

\newcommand{\vertiii}[1]{{\left\vert\kern-0.25ex\left\vert\kern-0.25ex\left\vert #1 
    \right\vert\kern-0.25ex\right\vert\kern-0.25ex\right\vert}}

\newcommand{\mexp}{\mathbb{E}}

\newcommand{\E}{\mathop{\mexp}}

\newcommand{\eps}{\varepsilon}

\usepackage{xcolor}

\newcommand{\beq}{\begin{eqnarray*}}
\newcommand{\eeq}{\end{eqnarray*}}
\newcommand{\beqn}{\begin{eqnarray}}
\newcommand{\eeqn}{\end{eqnarray}}

\usepackage{ifmtarg}
\usepackage{xifthen}%
\newcommand{\ent}[1][]{%
\ifthenelse{\isempty{#1}}{%
\mathrm{H}
}{
\mathrm{H}^{(#1)}
}}

\newcommand{\loch}[1][]{%
\ifthenelse{\isempty{#1}}{%
\mathrm{h}
}{
\mathrm{h}^{(#1)}
}}

\newcommand{\hide}[1]{}

\newcommand{\Acal}{\mathcal{A}}
\newcommand{\Bcal}{\mathcal{B}}
\newcommand{\Dcal}{\mathcal{D}}
\newcommand{\Fcal}{\mathcal{F}}

\newcommand{\Lcal}{\mathcal{L}}
\newcommand{\Ncal}{\mathcal{N}}

\newcommand{\Nbr}{\mathcal{N}_{[\,]}}

\newcommand{\NN}{\mathbb{N}}
\newcommand{\Lip}{\mathrm{Lip}}
\newcommand{\Hol}{\mathrm{H\ddot{o}l}}
\newcommand{\W}{\mathbb{W}}
\newcommand{\half}{\frac{1}{2}}

\newtheorem*{rep@theorem}{\rep@title}
\newcommand{\newreptheorem}[2]{%
\newenvironment{rep#1}[1]{%
 \def\rep@title{#2 \ref{##1}}%
 \begin{rep@theorem}}%
 {\end{rep@theorem}}}
\makeatother

\newreptheorem{proposition}{Proposition}

\title{\bf{Near-optimal learning with average H\"older smoothness}
}

\author{
 Steve Hanneke\thanks{\texttt{steve.hanneke@gmail.com}}
 \vspace{.05in}
 \\Purdue University
 \and Aryeh Kontorovich\thanks{\texttt{karyeh@cs.bgu.ac.il}
 }
 \vspace{.05in}
 \\Ben-Gurion University\\of the Negev
 \and Guy Kornowski\thanks{\texttt{guy.kornowski@weizmann.ac.il}}  \vspace{.05in}
 \\Weizmann Institute\\of Science
}

\date{}

\begin{document}

\maketitle
\begin{center}\vspace{-1cm}\today\vspace{0.5cm}\end{center}

\begin{abstract}
We generalize the notion of average Lipschitz smoothness proposed by \citet{ashlagi2021functions} by extending it to H\"older smoothness.
This measure of the ``effective smoothness'' of a function is sensitive to the underlying distribution and can be dramatically smaller than its classic ``worst-case'' H\"older constant.
We consider both the realizable and the agnostic (noisy) regression settings, proving upper and lower risk bounds in terms of the average H\"older smoothness; these rates improve upon both previously known rates even in the special case of average Lipschitz smoothness.
Moreover, our lower bound is tight in the realizable setting up to log factors, thus we establish the minimax rate.
From an algorithmic perspective, since our notion of average smoothness is defined with respect to the unknown underlying distribution, the learner does not have an explicit representation of the function class, hence is unable to execute ERM. Nevertheless, we provide distinct learning algorithms that achieve both (nearly) optimal learning rates.
Our results hold in any totally bounded metric space, and are stated in terms of its intrinsic geometry.
Overall, our results show that the classic worst-case notion of H\"older smoothness can be essentially replaced by its average, yielding considerably sharper guarantees.
\end{abstract}

\section{Introduction}

A fundamental theme throughout learning theory and statistics is that ``smooth'' functions ought to be easier to learn than ``rough'' ones --- an intuition that has been formalized and rigorously established in various frameworks %
\citep{MR1920390,tsybakov2008nonparametric,gine2021mathematical}.
H\"older continuity
is
a 
natural
and well-studied notion of smoothness  
that
measures the extent to which 
nearby
points can differ in function value
and includes
Lipschitz
continuity
as
an important special case.

These global moduli of smoothness, while convenient for theoretical analysis, suffer from the shortcoming of being overly pessimistic. Indeed, being distribution-independent, they fail to distinguish a function that is highly oscillatory everywhere from one that is smooth over most of the probability mass; see Figure \ref{fig:avg_smooth} for a simple illustration.
\begin{figure}
\begin{center}
	\includegraphics[trim=0cm 4cm 0cm 3cm,clip=true, width=0.9\textwidth]{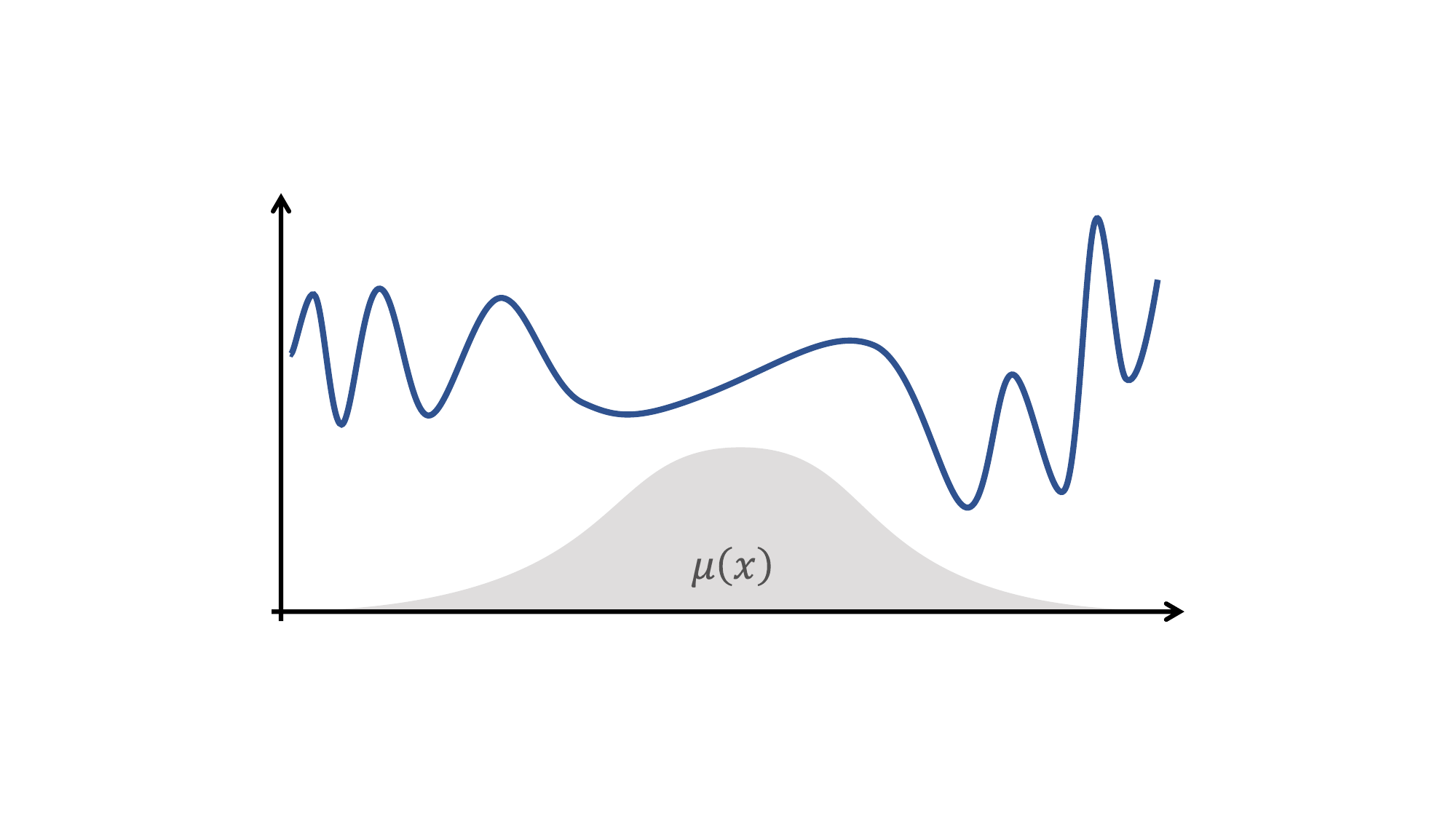}%
	\caption{Illustration of a function and a measure $\mu$
exhibiting a large gap between ``worst-case'' smoothness (occurring in low density regions) and average-smoothness with respect to $\mu$.}
	\label{fig:avg_smooth}
	\end{center}
\end{figure}
Moreover, classically studied classes of \emph{average smoothness} (e.g. Besov space) typically fix some distribution in advance (predominantly uniform), and then turn to consider smooth functions with respect to that single distribution. Thus, from a distribution-free statistical learning perspective --- where the underlying distribution is assumed to be unknown --- such classes fall short.

Seeking to address these drawbacks, \citet{ashlagi2021functions} proposed a natural notion of average Lipschitz smoothness with respect to a distribution.
Their average Lipschitz modulus can be considerably (even infinitely) smaller than the standard Lipschitz constant,
while still being able to control the excess risk. However, the risk bounds obtained by \citeauthor{ashlagi2021functions} are far from optimal, while the optimal rates for distribution-free learning of average smoothness classes remained unknown. In particular, the cost of adapting to the smoothness with respect to the underlying distribution (in contrast to using classic worst-case smoothness) remained unclear so far.

\paragraph{Our contributions.}

In this work, we generalize the aforementioned notion of average Lipschitz smoothness by extending it to H\"older smoothness of any exponent $\beta\in(0,1]$.
After formally defining the average H\"older smoothness of a function with respect to a distribution in \secref{subsec: avg smooth}, our contributions can be summarized as follows:

\begin{itemize}[leftmargin=*]
\item \textbf{Bracketing numbers upper bound (\thmref{thm: bracket}).}
We establish a nearly-optimal distribution-free bound on the bracketing entropy of our proposed average-smooth function class, serving as the main crux on which we base our analyses throughout the paper.
In particular, although it is known that asymptotically empirical covering numbers yield sharper bounds than bracketing numbers,\footnote{Uniform convergence of the $L_1(\mathcal{D})$ distance between the upper and lower bracket functions implies that, in the limit of sample size, the $\eps$-bracket functions are almost surely an empirical $(1+o(1))\eps$-cover; see \secref{sec: prelim} for a reminder of relevant definitions.}
in the case of average smoothness we reveal that the latter are tight up to a logarithmic factor.
    
\item \textbf{Realizable sample complexity (\thmref{thm: uc}).}
We derive a nearly-optimal sample complexity required for uniform convergence of average-H\"older functions in the realizable case, which was not previously known even in the special case of average Lipschitz functions.

\item \textbf{Optimal realizable learning algorithm (\thmref{thm: alg regression}).}
Since our notion of average smoothness
is defined with respect to the unknown sampling distribution, the learner does not have an explicit representation of the function class, and hence is unable to execute ERM.\footnote{Indeed, the learner cannot know for sure whether any given non-classically H\"older function belongs to the average-H\"older class.}
We are able to overcome this obstacle by constructing a realizable nonparametric regression algorithm with a nearly-optimal learning rate. Such a rate was not previously known even in the special case of average Lipschitz smoothness.

\item \textbf{Agnostic learning algorithm (\thmref{thm: agnostic alg}).}
We provide yet another learning algorithm for the fully agnostic (i.e. noisy) regression setting. Once again, our derived rate was not previously known even in the special case of average Lipschitz smoothness.

\item \textbf{Matching lower bound (\thmref{thm: lower}).} We prove a lower bound, showing that all the results mentioned above are tight up to logarithmic factors in the realizable case, establishing the (nearly) minimax risk rate for average-smooth classes.

\item \textbf{Illustrative comparisons (\secref{sec: examples}).}
Finally, we illustrate the extent to which the proposed smoothness notion is sharper than previously studied notions.
We provide examples in which the ``optimistic'' average-H\"older constant is infinitely apart from both its ``pessimistic'' worst-case counterpart, or even the average-Lipschitz ($\beta=1$) constant, exemplifying the substantial (possibly infinite) speed-ups in terms of learning rates.

\end{itemize}

\subsection{Related work.}

The sample complexities associated to distribution-free learning of (classic) H\"older classes is well covered in the literature, see for example the books by \citet{MR1920390,tsybakov2008nonparametric}.

Previous notions of average smoothness include
Bounded Variation (BV)
\citep{MR3156940}
in dimensions
one
and higher
\citep{MR0419394,MR2208697}.
One-dimensional BV
has found learning-theoretic applications
\citep{DBLP:journals/tit/BartlettKP97,DBLP:journals/iandc/Long04,MR1741038}, but to our knowledge the higher-dimensional variants have not.
Moreover, the positive results require $\mu$
to be 
uniformly distributed on a segment,
and the aforementioned results break down for more general measures
 --- especially if $\mu$ is not known to the learner.

Sobolev spaces, and 
the Sobolev-Slobodetskii norm
in particular
\citep{MR3287270},
bear some resemblance to our average H\"older smoothness.
However, \citet[Appendix I]{ashlagi2021functions} demonstrate
that
from a learning-theoretic perspective this notion is inadequate for
general (i.e., non-uniform or Lebesgue) measures, as it cannot be 
used to control
sample
complexity.
Results for controlling 
bracketing
in terms
of various measures of average smoothness
include
\citet{MR2324525}, who bound the bracketing numbers of Besov- and Sobolev-type classes
and
\citet{MR2761605}, 
who
used 
the
{\em averaged modulus of continuity} developed by \citet{MR995672};
again, these are all defined under the Lebesgue measure.
While it is easy to define these smoothness notions with respect to arbitrary distributions, we are not aware of any existing work to bound their corresponding sample complexity (or even their covering or bracketing numbers) in a distribution-independent manner.
Moreover, the smoothness notion studied in this paper is defined over arbitrary metric spaces, whereas previous notions are typically restricted to Euclidean structures (or variants thereof). Despite of the considerable generality of our setting, we are able to provide tight bounds for all metric spaces alike, without requiring specialized analyses.

A seminal work on recovering functions with
spatially inhomogeneous smoothness from noisy
samples is \citet{MR1635414}.
Arguably in the spirit of $\mu$-dependent H\"older smoothness,
some of the classic results on $k$-NN risk decay rates were refined by \citet{DBLP:journals/corr/ChaudhuriD14} via an
analysis that captures the interplay between the metric and the sampling distribution.
Another
related notion is
that of
{\em Probabilistic Lipschitzness} 
(PL)
\citep{urner:13a,urner:13,DBLP:conf/colt/KpotufeUB15},
which
seeks to relax a hard Lipschitz condition
on the 
regression function.
While PL is in the same spirit as our notion, one critical distinction from our work is that, while existing analyses of learning under PL 
have focused specifically on binary classification, our interest in the present work is learning real-valued functions.

As previously mentioned, the main feature 
setting
this work apart from
others 
studying regression under 
average smoothness is that
our notion is defined with respect to a {\em general, unknown}
measure $\mu$.
The notable exception is,
of course,
\citet{ashlagi2021functions}
--- 
who introduced the framework of efficiently learning smooth-on-average functions with respect to an unknown distribution. Although extending their definition from Lipschitz to H\"older average smoothness was straightforward, optimal minimax rates are likely inaccessible via their techniques, which relied on empirical covering numbers. Estimating the magnitude of these random
objects was a formidable challenge,
and \citeauthor{ashlagi2021functions}
were only able to do so to within an additive
error decaying with sample size;
this sampling noise appears to present
an inherent obstruction
to optimal rates.
Thus,
our results required
a novel technique
to overcome this
obstruction, which we did by tightly controlling the bracketing entropy.
Our H\"older-type extension is a direct adaptation of the
Pointwise Minimum Slope Extension (PMSE)
developed for the Lipschitz special case
by
\citeauthor{ashlagi2021functions},
which in turn is closely related to
the one introduced by
\citet{MR2431047}.

\section{Preliminaries.} \label{sec: prelim}

\paragraph{Setting.}
Throughout the paper we consider functions $f:\Omega\to[0,1]$ where $(\Omega,\rho)$ is a metric space. We will consider a distribution $\Dcal$ over $\Omega\times[0,1]$ with marginal $\mu$ over $\Omega$, such that $(\Omega,\rho,\mu)$ forms a metric probability space (namely, $\mu$ is supported on the Borel
$\sigma$-algebra induced by $\rho$).
We associate to any measurable function $f:\Omega\to[0,1]$ its $L_1$ risk $L_\Dcal(f):=\E_{(X,Y)\sim\Dcal}|f(X)-Y|$, and its empirical risk $L_S(f):=\frac{1}{n}\sum_{i=1}^{n}|f(X_i)-Y_i|$ with respect to a sample $S=(X_i,Y_i)_{i=1}^{n}\sim\Dcal^n$.
More generally, we associate to any measurable function its $L_1$ norm $\norm{f}_{L_1(\mu)}:=\E_{X\sim\mu}|f(X)|$, and given a sample $(X_1,\dots,X_n)\sim\mu^n$ we denote its $L_1$ norm with respect to the empirical measure $\norm{f}_{L_1(\mu_n)}:=\frac{1}{n}\sum_{i=1}^{n}|f(X_i)|$.

We say that a distribution $\Dcal$ over $\Omega\times[0,1]$ is {\em realizable} by a function class $\Fcal\subset[0,1]^\Omega$ if there exists
an $f^*\in\Fcal$ such that $L_{\Dcal}(f^*)=0$. Thus, $f^*(X)=Y$ almost surely, where $(X,Y)\sim\Dcal$.

\paragraph{Metric notions.}
The diameter of $A\subset\Omega$ is $\mathrm{diam}(A):=\sup_{x,x'\in A}\rho(x,x')$, and we denote by $B(x,r):=\{x'\in\Omega:\rho(x,x')\leq r\}$ the closed ball around $x\in\Omega$ of radius $r>0$.
For $t>0,~A,B\subset\Omega$, we say that $A$ is a $t$-\emph{cover} of $B$ if $B\subset\bigcup_{a\in A}B(a,t)$, and define the $t$-\emph{covering number} $\Ncal_B(t)$ to be the minimal cardinality of any $t$-cover of $B$.
We say that $A\subset B\subset \Omega$ is a $t$-\emph{packing} of $B$ if $\rho(a,a')\geq t$ for all $a\neq a'\in A$. We call $V$ a $t$-\emph{net} of $B$ if it is a $t$-cover and a $t$-packing. The induced \emph{Voronoi partition} of $B$ with respect to a net $V$ is its partitioning into subsets sharing the same nearest neighbor in $V$ (with ties broken in some consistent arbitrary manner).
A metric space $(\Omega,\rho)$ is said to be \emph{doubling} if there exists $d\in\NN$ such that every $r$-ball in $\Omega$ is contained in the union of some $d$ $r/2$-balls. The \emph{doubling dimension} is defined as $\min_{d\ge1}\log_2 d$ where the minimum is taken over $d$ satisfying the doubling property.

\paragraph{Bracketing.}

Given any two functions $l,u:\Omega\to[0,1]$, we say that $f:\Omega\to[0,1]$ belongs to the \emph{bracket} $[l,u]$ if $l\leq f\leq u$. A set of brackets $\Bcal$ is said to cover a function class $\Fcal$ if any function in $\Fcal$ belongs to some bracket in $\Bcal$.
We say that $[l,u]$ is a $t$-bracket with respect to a norm $\|\cdot\|$ if $\norm{u-l}\leq t$. The $t$-\emph{bracketing number} $\Nbr(\Fcal,\|\cdot\|,t)$ is defined as the minimal cardinality of any set of $t$-brackets that covers $\Fcal$. The logarithm of this quantity is called the \emph{bracketing entropy}.

\begin{remark}[Covering vs. bracketing] 
Having recalled two notions that quantify the ``size'' of a normed function space $(\Fcal,\|\cdot\|)$
--- namely, its 
covering 
and 
bracketing numbers
--- it is useful to note they are related through
\begin{equation} \label{eq: cov vs brack}
    \Ncal_{\Fcal}(\eps)\leq \Nbr(\Fcal,\|\cdot\|,2\eps)~,
\end{equation}
though no converse inequality of this sort holds in general. 
On the other hand, the main advantage of using bracketing numbers for generalization bounds is that it suffices to bound the {\em ambient} bracketing numbers with respect to the distribution-specific metric, as opposed to the {\em empirical covering} numbers which are necessary to guarantee generalization \citep[Section 2.1.1]{van1996weak}.
\end{remark}

\paragraph{Strong and weak mean.}

For any non-negative random variable $Z$ we define its \emph{weak mean} by $\W[Z]:=\sup_{t>0}t\Pr[Z\geq t]$, and note that $\W[Z]\leq\E[Z]$ by Markov's inequality. In the special case where $Z$ has finite support of size $N\geq 3$ where each atom has mass $1/N$ we have the reverse inequality $\E[Z]\leq 2\log(N)\W[Z]$ \citep[Lemma 22]{ashlagi2021functions}.

\subsection{Average smoothness.} \label{subsec: avg smooth}
For $\beta\in(0,1]$ and $f:\Omega\to\reals$, we define its $\beta$-slope at $x\in\Omega$ to be $\Lambda_f^\beta(x):=\sup_{y\in\Omega\setminus\{x\}}\frac{|f(x)-f(y)|}{\rho(x,y)^\beta}$.
Recall that $f$ is called $\beta$-H\"older continuous if $\norm{f}_{\Hol^\beta}:=\sup_{x\in\Omega}\Lambda_f^\beta(x)<\infty$, with this quantity serving as its H\"older seminorm. In particular, when $\beta=1$ these are exactly the Lipschitz functions equipped with the Lipschitz seminorm.
For a metric probability space $(\Omega,\rho,\mu)$, we consider the \emph{average} $\beta$-slope to be the mean of $\Lambda_f^\beta(X)$ where $X\sim\mu$. Namely, we define
\begin{align*}
\overline{\Lambda}_f^\beta(\mu)&:=\mathbb{E}_{X\sim\mu}[\Lambda_f^\beta(X)]
~,
\\
\widetilde{\Lambda}_f^\beta(\mu)&:=\W_{X\sim\mu}[\Lambda_f^\beta(X)]
=\sup_{t>0}\,t \cdot\mu(x:\Lambda_f^\beta(x)\geq t)
~.
\end{align*}
Notably,
\begin{equation} \label{eq:Lambda ineq}
\widetilde{\Lambda}_f^\beta(\mu)
\leq \overline{\Lambda}_f^\beta(\mu)
\leq \norm{f}_{\Hol^\beta}   ~,
\end{equation}
where each subsequent pair can be infinitely apart --- as we demonstrate in \secref{sec: examples}. Having defined notions of averaged smoothness, we can further define their corresponding function spaces
\begin{align*}
    \Hol^\beta_{L}(\Omega)&:=\left\{f:\Omega\to[0,1]\,:\,\norm{f}_{\Hol^\beta} \leq L\right\},
    \\
    \overline{\Hol}^\beta_{L}(\Omega,\mu)&:=\left\{f:\Omega\to[0,1]\,:\,\overline{\Lambda}_f^\beta(\mu)\leq L\right\},
    \\
    \widetilde{\Hol}^\beta_{L}(\Omega,\mu)&:=\left\{f:\Omega\to[0,1]\,:\,\widetilde{\Lambda}_f^\beta(\mu)\leq L\right\}.
\end{align*}
We occasionally omit $\mu$ when it is clear from context.  Note that $\Hol^\beta_{L}(\Omega)\subset\overline{\Hol}^\beta_{L}(\Omega,\mu)\subset\widetilde{\Hol}^\beta_{L}(\Omega,\mu)$ due to \eqref{eq:Lambda ineq}, where both containments are strict in general. 
The special case of $\beta=1$ recovers the 
average-Lipschitz
spaces $\Lip_L(\Omega)\subset\overline{\Lip}_L(\Omega,\mu)\subset\widetilde{\Lip}_L(\Omega,\mu)$ studied by \citet{ashlagi2021functions}.

\section{Generalization bounds}

Our first goal is to bound the bracketing entropy (namely, the logarithm of the bracketing number) of average-H\"older classes. We present this bound in full generality in terms of the underlying metric space, as captured by its covering number
(see \corref{cor: doubling uc} for the typical scaling of covering numbers). As we will soon establish, this bound
implies nearly-tight generalization guarantees in terms of the average smoothness constant.

\begin{theorem} \label{thm: bracket}
For any metric probability space $(\Omega,\rho,\mu)$, any $\beta\in(0,1]$ and any $0<\epsilon<L$, it holds
\begin{align*}
\log\Nbr(\overline{\Hol}^\beta_L(\Omega,\mu),L_1(\mu),\eps)
&\leq\log \Nbr(\widetilde{\Hol}^\beta_{L}(\Omega,\mu),L_1(\mu),\eps)
\\
&\leq
\Ncal_{\Omega}\left(\left(\frac{\eps}{128L\log(1/\eps)}\right)^{1/\beta}\right)\cdot\log\left(\frac{16\log_2(1/\eps)}{\eps}\right)
~.
\end{align*}
\end{theorem}

Crucially, the bound above does not depend on $\mu$, allowing us to obtain distribution-free generalization guarantees.
We defer the proof of \thmref{thm: bracket} to \appref{proof:bracket}.
We start by showing that bounding the bracketing entropy implies a generalization bound in the realizable case:

\begin{proposition} \label{prop: bracket to uc}
Suppose $(\Omega,\rho)$ is a metric space, $\Fcal\subseteq[0,1]^{\Omega}$ is a function class, and let $\Dcal$ be a distribution over $\Omega\times[0,1]$ which is realizable by $\Fcal$, with marginal $\mu$ over $\Omega$.
Then with probability at least $1-\delta$ over drawing a sample $S\sim\Dcal^n$ it holds that for all $f\in \Fcal:$

\begin{align*}
L_\Dcal(f)
\leq
1.01L_S(f)
+\inf_{\alpha\geq0}\left(\alpha+\frac{205\log \Ncal_{[\,]}(\Fcal,L_1(\mu),\alpha)}{n}\right)
+\frac{205\log(1/\delta)}{n}~.    
\end{align*}

\end{proposition}

\begin{remark}[Constant is arbitrary]
In \propref{prop: bracket to uc} and in what follows, the constant multiplying $L_S(f)$ is arbitrary, and can be replaced by $(1+\gamma)$ for any $\gamma>0$ at the expense of multiplying the remaining summands by $\gamma^{-1}$. In the next section we will provide a realizable regression algorithm that returns an approximate empirical risk minimizer $f$ for which $L_S(f)\approx0$,
thus this constant will not matter for our purposes.
\end{remark}

We prove \propref{prop: bracket to uc} in \appref{proof: bracket to uc}. By combining \thmref{thm: bracket} with \propref{prop: bracket to uc} and setting $\alpha=\eps/2$, we obtain the following realizable sample complexity result.

\begin{theorem} \label{thm: uc}
For any metric space $(\Omega,\rho)$,
any $\beta\in(0,1]$ and any $0<\epsilon<L$, let $\Dcal$ be a distribution over $\Omega\times[0,1]$ realizable by $\widetilde{\Hol}^\beta_{L}(\Omega,\mu)$.
Then there exists $N=N(\beta,\eps,\delta)\in \NN$
satisfying
\[
N=\widetilde{O}\left(
\frac{\Ncal_{\Omega}\left(\left(\frac{\eps}{256L\log(1/\eps)}\right)^{1/\beta}\right)+\log(1/\delta)}
{\eps}
\right)
\]
such that as long as $n\geq N$, with probability at least $1-\delta$ over drawing a sample $S\sim\Dcal^n$ it holds that for all $f\in \widetilde{\Hol}^\beta_{L}(\Omega,\mu):$
\[
L_{\Dcal}(f)\leq 1.01L_S(f)+\eps~.
\]
The same claim holds for the smaller class $\overline{\Hol}^\beta_L(\Omega,\mu)$.
\end{theorem}

\begin{corollary}[Doubling metrics] \label{cor: doubling uc}
In most cases of interest, $(\Omega,\rho)$ is a doubling metric space of some dimension $d$,\footnote{Namely, any ball of radius $r>0$ can be covered by $2^d$ balls of radius $r/2$.} e.g. when $\Omega$ is a subset of $\reals^d$ (or more generally a $d$-dimensional Banach space).
For $d$-dimensional doubling spaces of finite diameter we have
$\Ncal_{\Omega}(\eps)\lesssim\left(\frac{1}{\eps}\right)^d$ \citep[Lemma 2.1]{gottlieb2016adaptive}, which, plugged into \thmref{thm: uc}, yields the simplified sample complexity bound
\[
N=\widetilde{O}\left(\frac{L^{d/\beta}}{\eps^{(d+\beta)/\beta}}\right)~,
\]
or equivalently
\[
L_{\Dcal}(f)\leq 1.01L_S(f)+\widetilde{O}\left(\frac{L^{d/(d+\beta)}}{n^{\beta/(d+\beta)}}\right)~,
\]
up to a constant which depends
(exponentially)
on $d$, but is independent of $L,n$.
\end{corollary}

\begin{remark}[Tightness] \label{rem: tightness}
    The bounds in \thmref{thm: bracket} and \thmref{thm: uc} are both tight up to logarithmic factors,
    as we will prove in \secref{sec: lower}.
\end{remark}

\section{Realizable learning algorithm}
Recall that without knowing $\mu$, the underlying distribution over $\Omega$, we cannot know 
for sure whether a function $f$ belongs to
$\overline{\Hol}_{L}^\beta(\Omega,\mu)$
(except for the trivial case
$
f\in
\Hol^\beta_{L}(\Omega)
$).
This gives rise to the challenge of designing a
fully empirical
algorithm 
---
since standard empirical risk minimization is not possible. To that end, we provide the following algorithmic result with optimal guarantees (up to logarithmic factors).

\begin{theorem} \label{thm: alg regression}
For any metric space $(\Omega,\rho)$,
any $\beta\in(0,1]$ and any $0<\epsilon<L$, let $\Dcal$ be a distribution over $\Omega\times[0,1]$ realizable by $\overline{\Hol}^\beta_{L}(\Omega,\mu)$.
Then there exists a polynomial time learning algorithm $\Acal$, which, given a sample $S\sim\Dcal^n$ of size $n\geq N$ for some 
$N=N(\beta,\eps,\delta)\in\NN$
satisfying
\[
N=\widetilde{O}\left(
\frac{\Ncal_{\Omega}\left(\left(\frac{\eps}{256L\log(1/\eps)}\right)^{1/\beta}\right)+\log(1/\delta)}
{\eps}
\right),
\]
constructs a hypothesis $f=\Acal(S)$ such that 
$L_\Dcal(f)\leq\eps$
with probability at least
$1-\delta$.

\end{theorem}

\begin{remark}[Doubling metrics]
As mentioned in \corref{cor: doubling uc}, in most cases of interest we have
$\Ncal_{\Omega}(\eps)\lesssim\left(\frac{1}{\eps}\right)^d$. 
In that case, the algorithm above has sample complexity
\[
N=\widetilde{O}\left(\frac{L^{d/\beta}}{\eps^{(d+\beta)/\beta}}\right)~,
\]
or equivalently
\[
L_{\Dcal}(f)=\widetilde{O}\left(\frac{L^{d/(d+\beta)}}{n^{\beta/(d+\beta)}}\right)~,
\]
up to a constant which depends (exponentially) on $d$, but is independent of $L,n$.
\end{remark}

\begin{remark}[Computational complexity]
    The algorithm 
constructed in \thmref{thm: alg regression}
involves
a
one-time preprocessing step
after which
$f(x)$
can be evaluated at any
given
$x\in\Omega$
in $O(n^2)$ time.
We note that the computation at inference time matches that of (classic) Lipschitz/H\"older regression (e.g. \citealp{7944658}).
Furthermore, the computational complexity of the preprocessing step is similar to that in \citet[Theorem 7]{ashlagi2021functions} for the average Lipschitz case, where it is shown to run in time $\widetilde{O}(n^2)$. 
The complexity
analysis
of our prepossessing
step is entirely
analogous to
theirs, and we forgo
repeating it here.
\end{remark}

We will now outline the proof of \thmref{thm: alg regression}, which appears in Section~\ref{proof:alg}. The key idea is to analyze a natural fully-empirical quantity that will serve as an estimator of the true unknown average smoothness. To that end, given a sample $S=(X_i,Y_i)_{i=1}^{n}\sim\Dcal^n$ and a function $f:\Omega\to[0,1]$, consider the following quantity which can be established directly from the data:
\[
\widehat{\Lambda}^\beta_f:=\frac{1}{n}\sum_{i=1}^{n}\sup_{X_j\neq X_i}\frac{|f(X_i)-f(X_j)|}{\rho(X_i,X_j)^\beta}~.
\]
Namely, this is the empirical average smoothness with respect to the sampled points.
It would suit us well
if the empirical average smoothness
of a function did not greatly exceed its true average smoothness,
with high probability.
The fact
something like this turns out to be true is somewhat surprising and may be of independent interest:

\begin{repproposition}{prop: approx erm}(Informal)
Let $f^*:\Omega\to[0,1]$. Then with high probability $\widehat{\Lambda}_{f^*}^\beta\lesssim \overline{\Lambda}_{f^*}^\beta$ .
\end{repproposition}

The proposition above implies that restricting to the sample, and letting  $\widehat{f}(X_i):=Y_i$ yields a function over $\{X_1,\dots,X_n\}$ which is empirically average-smooth over the sample (with high probability). We then turn to show that any such function can be approximately extended to the whole space, in a way that guarantees its average smoothness with respect to the \emph{underlying distribution}.

\begin{repproposition}{prop: approx extend g}(Informal)
Let $\widehat{f}:\{X_1,\dots,X_n\}\to[0,1]$ where $(X_i)_{i=1}^{n}\sim\mu^n$. Then it is possible to construct $f:\Omega\to[0,1]$ such that with high probability $f(X_i)\approx\widehat{f}(X_i)$ for all $i\in[N]$, and $\overline{\Lambda}^\beta_f(\mu)\lesssim \widehat{\Lambda}^\beta_{\widehat{f}}$ .
\end{repproposition}

We will now sketch the procedure described in \propref{prop: approx extend g}, which serves as the main challenge in proving \thmref{thm: alg regression}. Roughly speaking, the algorithm sorts the sampled points with respect to their relative slope to one another. Then, it discards a fraction of the sampled points with largest relative slope, which can be thought of as ``outliers''. Then, the algorithm proceeds to extend the function in a smooth fashion among the remaining ``well-behaved'' samples. A careful probabilistic analysis shows that disregarding just the right amount of samples induces small error, while being average-smooth with high probability.

Overall this procedure yields a function $f:\Omega\to[0,1]$ which is an approximate empirical-minimizer (since $f(X_i)\approx \widehat{f}(X_i)=Y_i$), while guaranteed to be averagely-smooth with respect to the unknown distribution.
Thus we can apply the uniform convergence of \thmref{thm: uc}, proving \thmref{thm: alg regression}.

\section{Agnostic learning algorithm}

Noticeably, up to this point, both the uniform convergence result we derived (\thmref{thm: uc}) as well as the algorithmic result (\thmref{thm: alg regression}) are tailored for the realizable regression setting. Inspired by a recent result of \citet{hopkins2022realizable} that showed a reduction from agnostic learning to realizable learning, we provide an algorithm for agnostic (i.e. noisy) regression of average-smooth functions.
It is worth noting that the following algorithm does not require any prior assumption on the noise model, unlike many nonparametric regression methods, due to our distribution free analysis.

\begin{theorem} \label{thm: agnostic alg}
There exists a learning algorithm $\Acal$ such that for any metric space $(\Omega,\rho)$,
any $\beta\in(0,1],~0<\epsilon<L$, and any distribution $\Dcal$ over $\Omega\times[0,1]$, given a sample $S\sim\Dcal^n$ of size $n\geq N$ for some 
$N=N(\beta,\eps,\delta)$
satisfying
\[
N=\widetilde{O}\left(
\frac{\Ncal_{\Omega}\left(\left(\frac{\eps}{640L\log(1/\eps)}\right)^{1/\beta}\right)+\log(1/\delta)}
{\eps^2}
\right),
\]
the algorithm constructs a hypothesis $f=\Acal(S)$ such that 
$L_\Dcal(f)\leq \inf_{f^*\in\overline{\Hol}_{L}^\beta(\Omega,\mu)}L_\Dcal(f^*)+\epsilon$
with probability at least
$1-\delta$.

\end{theorem}

\begin{remark}[Doubling metrics]
As mentioned in \corref{cor: doubling uc}, in most cases of interest we have
$\Ncal_{\Omega}(\eps)\lesssim\left(\frac{1}{\eps}\right)^d$. 
In that case, the algorithm above has sample complexity 
\[
N=\widetilde{O}\left(\frac{L^{d/\beta}}{\eps^{(d+2\beta)/\beta}}\right)~,
\]
or equivalently
\[
L_{\Dcal}(f)=
\inf_{f^*\in\overline{\Hol}_{L}^\beta(\Omega,\mu)}L_\Dcal(f^*)
+\widetilde{O}\left(\frac{L^{d/(d+2\beta)}}{n^{\beta/(d+2\beta)}}\right)~,
\]
up to a constant which depends
(exponentially) on $d$, but
is
independent of $L,n$.
\end{remark}

Though our agnostic algorithm is similar in spirit to that obtained by the reduction of \citet{hopkins2022realizable}, our analysis is self-contained and crucially relies on the bracketing bound given by \thmref{thm: bracket}, as well as analyzing the empirical smoothness estimator as provided by \propref{prop: approx erm}.
We also note that unlike our algorithm for realizable learning, the agnostic algorithm is not computationally efficient. This seems to be inherent for such reductions, and we do not know whether this blow-up in running time can be avoided or not.

We will now describe the proof of \thmref{thm: agnostic alg} which appears in Section~\ref{sec: agnostic alg proof}.
Given a sample $S$ of size $n$, consider dividing it into two sub-samples $S_1\cup S_2=S$ of size $n/2$ each. We first use $S_1$ in order to construct an empirical $\epsilon$-net $h_1,\dots,h_N:S_1\to[0,1]$, namely a set of functions which are sufficiently empirically smooth over the sample, yet far away enough from one another when averaged over the sample. Recalling that bracketing numbers upper bound covering numbers (\eqref{eq: cov vs brack}), and since \thmref{thm: bracket} holds true for every measure (in particular for the empirical measure), we can bound $\log N\lesssim\Ncal_{\Omega}((\epsilon/L)^{1/\beta})$.
Moreover, using \propref{prop: approx erm} we know that $f^*:=\arg\min_{f\in\overline{\Hol}_{L}^\beta(\Omega,\mu)}L_\Dcal(f)$ is likely to be $\widetilde{O}(L)$ average-smooth over $S_1$, so there must exist some $h_j$ with $\epsilon$ excess empirical loss (since $f^*$ is in the class we are $\epsilon$-covering). Thus running the realizable algorithm of \thmref{thm: alg regression} over all $\{h_1,\dots,h_N\}$, producing $f_1,\dots,f_N:\Omega\to[0,1]$, yields at least one function which has both small excess empirical error, while being smooth with respect to the underlying distribution.
Finally, running ERM over $\{f_1,\dots,f_N\}$ with respect to the fresh sample $S_2$ reveals such a good candidate function within $\frac{\log(N)+\log(1/\delta)}{\epsilon^2}$ samples by applying standard uniform convergence for finite classes (i.e. Hoeffding's inequality with the union bound).

\section{Lower bound} \label{sec: lower}

 We
 now turn to show that the bounds proved in \thmref{thm: bracket}, \thmref{thm: uc} and \thmref{thm: alg regression} are all tight up to logarithmic factors. In fact, since the bracketing entropy bound of \thmref{thm: bracket} implies the generalization bound of \thmref{thm: uc}
 and the latter implies the sample complexity in \thmref{thm: alg regression}, it is enough to show that the latter is nearly optimal.

\begin{theorem} \label{thm: lower}
    For any $\beta\in(0,1],~\eps\in(0,1)$ any metric space $(\Omega,\rho)$ and $L\geq\frac{8}{\mathrm{diam}(\Omega)}$, there exists a distribution $\Dcal$ over $\Omega\times[0,1]$ which is realizable by $\overline{\Hol}^\beta_L(\Omega)$ such that any learning algorithm that produces $f=A(S)$ with $L_{\Dcal}(f)\leq\eps$ with constant probability, must have sample complexity 
    \[
    n=\Omega\left(\frac{\Ncal_{\Omega}((\eps/L)^{1/\beta})}{\eps}\right)~.
    \]
\end{theorem}

\begin{remark}[Typical case]
    In most cases of interest it holds that $\Ncal_{\Omega}(\eps)\gtrsim\left(\frac{1}{\eps}\right)^d$ for some constant $d$, e.g. when $\Omega$ is a subset of non-empty interior in $\reals^d$ (or more generally in any $d$-dimensional Banach space).\footnote{Note that assuming a subset has nonempty interior implies that it cannot be isometrically embedded to a lower dimensional space. Hence, this $d$ encapsulates the ``true'' intrinsic metric dimension.} That being the case, \thmref{thm: lower} yields the simplified sample complexity lower bound of
    \[
    n=\Omega\left(\frac{L^{d/\beta}}{\eps^{(d+\beta)/\beta}}\right)
    \]
    Equivalently, we obtain an excess risk lower bound of
    \[
    L_{\Dcal}(f)=\Omega\left(\frac{L^{d/(d+\beta)}}{n^{\beta/(d+\beta)}}\right)~.
    \]  
\end{remark}

We will now provide a proof sketch of \thmref{thm: lower}, while the full proof appears in Section~\ref{proof:lower}.
Suppose $K\subset\Omega$ is a $(\eps/L)^{1/\beta}$-net of \emph{most} of $\Omega$, yet $x_0\in\Omega$ is some ``isolated'' point at constant distance away from $K$ (we show that such $x_0,K$ always exist). Let $\mu$ be the measure that assigns $1-\eps$ probability mass to $x_0$, while the rest of the probability mass is distributed uniformly over $K$.
Now consider a (random) function that independently assigns either $0$ or $1$ to each point in $K$ uniformly, and is constant over $x_0$. Since points in $K$ are $(\eps/L)^{1/\beta}$ away from one another, the local $\beta$-slope at each point in $K$ is roughly $1/((\eps/L)^{1/\beta})^\beta=L/\eps$, while the slope at $x_0$ is small since it is far enough from other points.
Averaging over the space with respect to $\mu$, we see that the function is $\mu(K)\cdot L/\eps=L$ average-H\"older.
Now, we imitate the standard lower bound proof for VC classes over $K$: Since any point in $K$ is sampled with probability $\eps/|K|$, any learning algorithm with much fewer than $|K|/\eps\approx\Ncal_{\Omega}((\eps/L)^{1/\beta})/\eps$ examples will guess wrong a large portion of $K$, suffering $L_1$-loss of at least order of $\mu(K)=\eps$.


\section{Illustrative examples} \label{sec: examples}

Having established the control that average-H\"older smoothness has on generalization, we illustrate the vast possible gap between the average smoothness and it's ``worst-case'' classic counterpart. Indeed, in the examples we provide, the gap is infinite.
Moreover, we also show that classes of average-H\"older smoothness are significantly richer than the previously studied average-Lipschitz, %
motivating
the more general H\"older framework considered in this work. Finally, it is illuminating to notice that both claims to follow actually consist of the same simple function $f(x)=\boldsymbol{1}[x>\half]$ though with respect to different distributions, emphasizing the crucial role of the underlying distribution in terms of establishing the function classes.

\begin{claim} \label{example 1}
For any $L>0,\beta\in(0,1)$, there exist $f:\Omega\to[0,1]$ and a probability measure $\mu$ such that
\begin{itemize}
    \item $f$ is average-H\"older: $f\in\overline{\Hol}^\beta_L(\Omega,\mu)$ .
    \item $f$ is not H\"older with any finite H\"older constant: For all $M>0:~f\notin{\Hol}_{M}^\beta(\Omega)$ .
    \item $f$ is not (even) weakly-average-Lipschitz with any finite modulus:
    For all $M>0:f\notin\widetilde{\Lip}_{M}(\Omega,\mu)$ .
\end{itemize}
Thus, $\overline{\Hol}^\beta_L(\Omega,\mu)\not\subset \bigcup_{M=0}^{\infty}\left({\Hol}_{M}^\beta(\Omega)\cup\widetilde{\Lip}_{M}(\Omega,\mu)\right)$.
\end{claim}

\begin{claim} \label{example 2}
For any $L>0,\beta\in(0,1)$, there exist $f:\Omega\to[0,1]$ and a probability measure $\mu$ such that
\begin{itemize}
    \item $f$ is weakly-average-H\"older: $f\in\widetilde{\Hol}^\beta_L(\Omega,\mu)$ .
    \item $f$ is not strongly-average-H\"older with any finite modulus: For all $M>0:~f\notin\overline{\Hol}_{M}^\beta(\Omega)$ .
    \item $f$ is not (even) weakly-average-Lipschitz with any finite modulus:
    For all $M>0:f\notin\widetilde{\Lip}_{M}(\Omega,\mu)$ .
\end{itemize}
Thus, $\widetilde{\Hol}^\beta_L(\Omega,\mu)\not\subset \bigcup_{M=0}^{\infty}\left(\overline{\Hol}_{M}^\beta(\Omega,\mu)\cup\widetilde{\Lip}_{M}(\Omega,\mu)\right)$.
\end{claim}

We prove both of the claims above in Section~\ref{proof:examples}.

\section{Proofs} \label{sec: proofs}

\subsection{Proof of \thmref{thm: bracket}} \label{proof:bracket}

We start by stating a strengthened version of the triangle inequality (also known as the ``snowflake'' triangle inequality) which we will use later on. For any $\beta\in(0,1],~x\neq y,z\in\Omega$:
\begin{equation} \label{eq: beta triangle ineq}
\rho(x,y)^\beta \leq \rho(x,z)^\beta+\rho(z,y)^\beta
~.    
\end{equation}
Indeed, this follows from
\begin{align*}
\frac{\rho(x,z)^\beta+\rho(z,y)^\beta}{\rho(x,y)^\beta}
&\geq\frac{\rho(x,z)^\beta+\rho(z,y)^\beta}{(\rho(x,z)+\rho(z,y))^\beta}
=\left(\frac{\rho(x,z)}{\rho(x,z)+\rho(z,y)}\right)^\beta+\left(\frac{\rho(z,y)}{\rho(x,z)+\rho(z,y)}\right)^\beta
\\&
\geq\left(\frac{\rho(x,z)}{\rho(x,z)+\rho(z,y)}\right)+\left(\frac{\rho(z,y)}{\rho(x,z)+\rho(z,y)}\right)=1
~.
\end{align*}
Let $0<\eps<\frac{1}{4}$, denote 
$K:=\lceil\log_2(1/\eps)\rceil,~\eps':=\frac{1}{(K+1)2^K}$ and note that
\begin{equation} \label{eq: epsilon' lower bound epsilon}
\eps'
\geq\frac{1}{\left(\log_2(1/\eps)+2\right)2^{\log_2(1/\eps)+1}}
=\frac{\eps}{2\left(\log_2(1/\eps)+2\right)}
\geq\frac{\eps}{4\log_2(1/\eps)}
~.    
\end{equation}
Let $N=\{x_1,\dots,x_{|N|}\}$ be a $\left(\frac{\eps'}{32L}\right)^{1/\beta}$-net of $\Omega$ of size $|N|=\Ncal_{\Omega}\left(\left(\frac{\eps'}{32L}\right)^{1/\beta}\right)$, and let $\Pi=\{C_1,\dots,C_{|N|}\}$ be its induced Voronoi partition. We define $\Bcal=\{[l_j,u_j]\}_{j\in J}\subset[0,1]^{\Omega}\times[0,1]^{\Omega}$
to be the pairs of functions constructed as follows:

\begin{itemize}
    \item $l,u$ are both constant over every cell $C_i\in\Pi$, and map each cell to a value in $\{0,\frac{\eps'}{2},\eps',\frac{3\eps'}{2},\dots,1\}$.
    \item Choose some cells $S_1\subset \Pi$ such that $\mu(\bigcup_{C_i\in S_1}C_i)\leq\eps'$ and set for any $C_i\in S_1:\ l|_{C_i}=0,~u|_{C_i}=1$.
    \item For $m=2,\dots,K$ choose some ``unchosen'' cells $S_m\subset \Pi\setminus\bigcup_{j<m}S_j$ such that $\mu(\bigcup_{C_i\in S_m}C_i)\leq2^{m-1}\eps'$ and set for any $C_i\in S_m:\ l|_{C_i}\in\{0,\frac{1}{2^m},\frac{2}{2^m},\dots,\frac{2^m-2}{2^m}\},~,~u|_{C_i}=l+\frac{1}{2^{m-1}}$.
    \item In the ''remaining'' cells $S_{K+1}:= \Pi\setminus\bigcup_{j\leq K}S_j$ set for any $C_i\in S_{K+1}:$  \[
    l|_{C_i}\in\left\{0,\frac{1}{2^{K+1}},\frac{2}{2^{K+1}},\dots,\frac{2^{K+1}-2}{2^{K+1}}\right\},
    ~u|_{C_i}=l+\frac{1}{2^{{K}}}
    ~.
    \]
\end{itemize}
Notice that for any $[l,u]\in \Bcal$ we have
\begin{align*}
\norm{l-u}_{L^1(\mu)}
&=\sum_{C_i\in \Pi}\int_{C_i}|l(x)-u(x)|d\mu(x)
=\sum_{m=1}^{K+1}\sum_{C_i\in S_m}\int_{C_i}|l(x)-u(x)|d\mu(x)
\\&=\sum_{m=1}^{K+1}\sum_{C_i\in S_m}\int_{C_i}\frac{1}{2^{m-1}}d\mu(x)
=\sum_{m=1}^{K+1}\frac{1}{2^{m-1}}\sum_{C_i\in S_m}\mu(C_i)
\\&=\sum_{m=1}^{K+1}\frac{2^{m-1}\eps'}{2^{m-1}}
=\eps'(K+1)
=\frac{1}{2^K}
\leq\eps~.
\end{align*}
Furthermore, we can bound $|\Bcal|$ by noticing that any such $l$ is defined by its values over $|N|$ cells who all belong to $\{0,\frac{\eps'}{2},\eps',\dots,1\}$, and that once $l$ is fixed then any associated $u$ has at most $K+1$ possible values over each cell since it equals $l+\frac{1}{2^{m-1}}$ for some $m\in[K+1]$. Thus
\[
|\Bcal|\leq (K+1)\left(\frac{8}{\eps'}\right)^{|N|}
\leq\log_2\left(\frac{1}{\eps}\right)\cdot\left(\frac{16\log_2(1/\eps)}{\eps}\right)^{\Ncal\left(\left(\frac{\eps}{128L\log(1/\eps)}\right)^{1/\beta}\right)}
~,
\]
where the last inequality uses \eqref{eq: epsilon' lower bound epsilon} and definition of $K$.
In order to finish the proof, in remains to show that $\Bcal$ indeed cover $\widetilde{\Hol}_L^\beta(\Omega,\mu)$ as brackets. Namely, that for any $f\in\widetilde{\Hol}_L^\beta(\Omega,\mu)$ there exist $[l,u]\in \Bcal$ such that $l\leq f\leq u$.
To that end, let $f\in\widetilde{\Hol}_L^\beta(\Omega,\mu)$. Denote
\[
S_1^f:=\left\{C_i\in\Pi\mathrm{~:~}\forall x\in C_i:\Lambda_{f}^\beta(x)\geq \frac{L}{\eps'}\right\}
\]
and notice that $\bigcup\{C_i\in S_1^f\}\subseteq\{x:\Lambda_{f}^\beta(x)\geq L/\eps'\}\implies \mu(\bigcup\{C_i\in S_1^f\})\leq \eps'$.
Hence we can pick $[l,u]\in \Bcal$ such that $(l|_{C_i},u|_{C_i})\equiv(0,1)$ for any $C_i\in S^{f}_1$ (serving as $S_1$ in their construction). Clearly any such $l,u$ bound $f$ over these cells. Furthermore, for $m=2,\dots,K$ we denote
\[
S_m^f:=\left\{C_i\in\Pi\setminus\bigcup_{j<m}S_j^f\mathrm{~:~}\forall x\in C_i:\Lambda_f^\beta(x)\geq\frac{L}{2^{m-1}\eps'}
~\right\}~,
\]
and notice that $\bigcup \{C_i\in S_m^f\}\subseteq\{x:\Lambda_{f}^\beta(x)\geq L/(2^{m-1}\eps')\}\implies \mu(\bigcup \{C_i\in S_m^f\})\leq 2^{m-1}\eps'$. Consequently we can let $S^{f}_m$ serve as $S_m$ in the construction of $[l,u]\in\Bcal$, assuming we will show such a choice can serve as a bracket of $f$ over such cells. Indeed, for any $x\in C_i$ we have
\[
|f(x)-f(z_i)|
\leq\Lambda_f^\beta(z_i)\cdot\rho(x,z_i)^\beta
\overset{\eqref{eq: beta triangle ineq}}{\leq}
\frac{L}{2^{m-2}\eps'}\cdot\frac{2\eps'}{32L}
=\frac{1}{2^{m+2}},
\]
which by the triangle inequality shows in particular that for any $x,y\in C_i:$
\[
|f(x)-f(y)|\leq|f(x)-f(z_i)|+|f(z_i)-f(y)|
\leq
\frac{1}{2^{m+1}}=\frac{1}{4\cdot2^{m-1}}
~.
\]
So clearly there exists $\alpha_i\in\{0,\frac{1}{2^m},\frac{2}{2^m},\dots,\frac{2^m-2}{2^m}\}$ such that
$\alpha_i\leq f|_{C_i}\leq\alpha_i+\frac{1}{2^{m-1}}$, and by setting $l|_{C_i},u|_{C_i}=(\alpha_i,\alpha_i+\frac{1}{2^{m-1}})$ for any $C_i\in S^{f}_m$  we ensure the bracketing property. Finally, for any of the remaining cells $S_{K+1}^{f}:=\Pi\setminus\bigcup_{j\leq K}S_j^{f}$ we get by construction that $\exists z_i\in C_i:\Lambda_f^\beta(z_i)< \frac{L}{2^{K}\eps'}$ (otherwise they would satisfy the condition for some previously constructed $S_m^f$). Hence for any $x\in C_i$ we have
\[
|f(x)-f(z_i)|\leq \Lambda_f^\beta(z_i)\cdot\rho(x,z_i)^\beta
\overset{\eqref{eq: beta triangle ineq}}{\leq} \frac{L}{2^{K}\eps'}\cdot\frac{2\eps'}{32L}
=\frac{1}{2^{K+4}},
\]
which by the triangle inequality shows that for any $x,y\in C_i:$
\[
|f(x)-f(y)|\leq\frac{1}{2^{K+3}}=\frac{1}{8\cdot2^{K}}
~.
\]
So as before, there clearly exists $\alpha_i\in\{0,\frac{1}{2^{K+1}},\frac{2}{2^{K+1}},\dots,\frac{2^{K+1}-2}{2^{K+1}}\}$ such that
$\alpha_i\leq f|_{C_i}\leq\alpha_i+\frac{1}{2^{K}}$, and by setting $l|_{C_i},u|_{C_i}=(\alpha_i,\alpha_i+\frac{1}{2^{K}})$ for any $C_i\in S^{f}_m$  we ensure the bracketing property over all of $\Omega$, which finishes the proof.

\subsection{Proof of \propref{prop: bracket to uc}} \label{proof: bracket to uc}

Recalling that the realizability assumption ensures a ``perfect'' predictor $f^*\in\Fcal$, we start by introducing the loss class $\Lcal_{\Fcal}\subset[0,1]^\Omega:$
\[
\Lcal_{\Fcal}=\{\ell_f(x):=|f(x)-f^*(x)|:f\in\Fcal\}~.
\]
Fix $\alpha>0$. We observe that $\Lcal_{\Fcal}$ is no larger than $\Fcal$ in terms of bracketing entropy, namely
\begin{equation} \label{eq: bracket loss}
\Nbr(\Lcal_{\Fcal},L_1(\mu),\alpha)\leq\Nbr(\Fcal,L_1(\mu),\alpha)~.    
\end{equation}
Indeed, suppose we are given an $\alpha$-bracketing of $\Fcal$ denoted by $\Bcal_{\alpha}$, and denote for any $f\in\Fcal$ by $[l_f,u_f]\in\Bcal_{\alpha}$ its associated bracket. Then any $\ell_f\in\Lcal_\Fcal$ is inside the bracket $[l_{\ell_f},u_{\ell_f}]$ where
\begin{align*}
    l_{\ell_f}:=&~\max\{0\,,\,\min\{l_f-f^*,f^*-u_f\}\}~,
    \\
    u_{\ell_f}:=&~\min\{1\,,\,\max\{u_f-f^*,f^*-l_f\}\}
    ~.
\end{align*}
It is straightforward to verify that $\norm{u_{\ell_f}-l_{\ell_f}}_{L_1(\mu)}\leq\norm{u_f-l_f}_{L_1(\mu)}\leq\alpha$, and clearly the set of all such brackets is of size at most $|\Bcal_{\alpha}|$, yielding \eqref{eq: bracket loss}.

Now notice that for any $f\in\Fcal:$
\[
L_{\Dcal}(f)-1.01L_S(f)
=\norm{\ell_f}_{L_1(\mu)}-1.01\norm{\ell_f}_{L_1(\mu_n)}
\leq \alpha+\norm{l_{\ell_f}}_{L_1(\mu)}-1.01\norm{l_{\ell_f}}_{L_1(\mu_n)}~,
\]
hence
\begin{align} \label{eq: max RHS}
\sup_{f\in\Fcal}\left(L_\Dcal(f)-1.01L_S(f)\right)
    \leq \alpha+
    \max_{l_{\ell_f}}(\norm{l_{\ell_f}}_{L_1(\mu)}-1.01\norm{l_{\ell_f}}_{L_1(\mu_n)})
    ~.
\end{align}

In order to bound the right hand side, fix some $l_{\ell_f}$, and note that $\mathrm{Var}(l_{\ell_f})\leq\norm{l_{\ell_f}^2}_{L_1(\mu)}\leq\norm{l_{\ell_f}}_{L_1(\mu)}$, since $l_{\ell_f}(x)\in[0,1]$.
Thus by Bernstein's inequality and the AM-GM inequality we get that with probability at least $1-\gamma:$
    \begin{align*}
    \norm{l_{\ell_f}}_{L_1(\mu)}-\norm{l_{\ell_f}}_{L_1(\mu_n)}
    \leq~ &\frac{\log(1/\gamma)}{n}+\sqrt{\frac{2\norm{l_{\ell_f}}_{L_1(\mu)}\log(1/\gamma)}{n}}
    \\
    \leq~
    &\frac{202\log(1/\gamma)}{n}+\frac{1}{101}\norm{l_{\ell_f}}_{L_1(\mu)}
    \\
    \implies
    \norm{l_{\ell_f}}_{L_1(\mu)}-1.01\norm{l_{\ell_f}}_{L_1(\mu_n)}
    \leq~& \frac{205\log(1/\gamma)}{n}~.
    \end{align*}
Setting $\gamma=\delta/\Nbr(\Fcal,L_1(\mu),\alpha)$ and taking a union bound over $l_{\ell_f}$ whose number is bounded due to \eqref{eq: bracket loss}, we see that with probability $1-\delta:$
\[
\max_{l_{\ell_f}}(\norm{l_{\ell_f}}_{L_1(\mu)}-1.01\norm{l_{\ell_f}}_{L_1(\mu_n)})
\leq \frac{205\log\Nbr(\Fcal,L_1(\mu),\alpha)+205\log(1/\delta)}{n}~.
\]
Plugging this back into \eqref{eq: max RHS}, and minimizing over $\alpha>0$ finishes the proof.

\subsection{Proof of \thmref{thm: alg regression}} \label{proof:alg}

\begin{proposition} \label{prop: approx erm}
    Let $f:\Omega\to[0,1]$. Then with probability at least $1-\delta/2$ over drawing a sample it holds that 
    \[
    \widehat{\Lambda}_{f}^\beta
    \leq 4\log^2(4n/\delta)\overline{\Lambda}_f^\beta(\mu)+\frac{4\log^2(4n/\delta)}{n}~.
    \]
\end{proposition}

\begin{corollary} \label{cor: approx erm}
    If $\Dcal$ is realizable by $\overline{\Hol}^\beta_{L}(\Omega,\mu)$, then for $f^*:\Omega\to[0,1]$ such that $L_{\Dcal}(f^*)=0$ it holds with probability at least $1-\delta/2:~\widehat{\Lambda}^\beta_{f^*}\leq 5\log^2(4n/\delta)L$. Hence,
    $\widehat{f}(X_i):=f^*(X_i)=Y_i$ satisfies $L_S(\widehat{f})=0$ and $\widehat{\Lambda}^\beta_{\widehat{f}}\leq 5\log^2(4n/\delta)L$ .
\end{corollary}

\begin{proof}(of \propref{prop: approx erm})
Fix $f:\Omega\to[0,1]$. Given a sample $(X_i)_{i=1}^{n}\sim\mu^n$ which induces an empirical measure $\mu_n$, we get
\begin{equation} \label{eq: hat_Lambda<log(n)W}
\widehat{\Lambda}_f^\beta
\leq \frac{1}{n}\sum_{i=1}^{n}\sup_{z\neq X_i}\frac{|f(X_i)-f(z)|}{\rho(X_i,z)^\beta}
=\E_{X\sim \mu_n}[\Lambda_f^\beta(X)]
\leq 2\log(n)\W_{X\sim \mu_n}[\Lambda_f^\beta(X)]
~,
\end{equation}
where the last inequality follows from the reversed strong-weak mean inequality for uniform measures.
We will now show that with high probability $\W_{X\sim \mu_n}[\Lambda_f^\beta(X)]\lesssim \W_{X\sim \mu}[\Lambda_f^\beta(X)]=\widetilde{\Lambda}_f^\beta$. To that end, we denote for any $t>0:\,M_f(t):=\{x:\Lambda_f^\beta(x)\geq t\}\subset\Omega$, let $K:=\widetilde{\Lambda}_f^\beta(\mu),~N:=\lceil2\log(4n/\delta)\log\log(4n/\delta)\rceil$ and note that
\begin{align} \label{eq: 3 sups}
\W_{X\sim \mu_n}[\Lambda_f^\beta(X)]
&=\sup_{t>0}t\mu_n(M_f(t)) 
\\
&\leq \sup_{0<t\leq K}t\mu_n(M_f(t))
+2\max_{j\in\{0,1,\dots,N-1\}}2^j K\mu_n(M_f(2^j K))
+\sup_{t\geq 2^N K}t\mu_n(M_f(t))~.     \nonumber
\end{align}
We will bound all three summands above. We easily bound the first term by
\begin{equation} \label{eq: sup1}
\sup_{0<t\leq K}t\mu_n(M_f(t))\leq K\cdot 1=\widetilde{\Lambda}_f^\beta(\mu)~.    
\end{equation}
For the second term, denote for any $t>0$ by $M_f^{+}(t)\supset M_f(t)$ a containing set for which $\frac{1}{n}\leq\mu(M_f^{+}(t))\leq\mu(M_f(t))+\frac{1}{n}$. We can always assume without loss of generality that such a set exists.\footnote{Such a set does not exist only in the case of atoms 
$x_0\in\Omega$
with large probability mass $\mu(x_0)$. If that is the case,
consider a ``copy'' metric space $\widetilde{\Omega}$ with $x_0$ split into two points $x_1,x_2\in\widetilde{\Omega}$ at distance $\eps$ apart and each of mass $\mu(x_0)/2$.
Any function $f:\Omega\to\R$
is extended to 
$\widetilde{f}:\widetilde{\Omega}\to\R$
via $\widetilde{f}(x_1)=\widetilde{f}(x_2)=f(x_0)$. 
Repeating the split if necessary and
taking $\eps\downarrow0$, we obtain
a space $\widetilde{\Omega}$
with all of the relevant properties of
$\Omega$ but no atoms of large mass.
} By the multiplicative Chernoff bound we have for any $t,\alpha>0:$
\[
\Pr_{S}\left[\mu_n(M_f^{+}(t))\geq(1+\alpha)\mu(M_f^{+}(t))\right]\leq\frac{e^\alpha}{(1+\alpha)^{1+\alpha}}~,
\]
hence by the union bound we get with probability at least $1-\frac{Ne^\alpha}{(1+\alpha)^{1+\alpha}}:$
\begin{align*}
\max_{j\in\{0,1,\dots,N-1\}}2^j K\mu_n(M_f(2^j K))
&\leq
\max_{j\in\{0,1,\dots,N-1\}}2^j K\mu_n(M_f^+(2^j K))
\\
&\leq(1+\alpha)\max_{j\in\{0,1,\dots,N-1\}}2^j L\mu(M_f^+(2^j K))
\\
&\leq(1+\alpha)\max_{j\in\{0,1,\dots,N-1\}}2^j K\left(\mu(M_f(2^j K))+\frac{1}{n}\right)
\\
&\leq(1+\alpha)\widetilde{\Lambda}^\beta_f(\mu)+\frac{1+\alpha}{n}~.
\end{align*}
Letting $\alpha=\log(4n/\delta)-1$, by our choice of $N=\lceil2\log(4n/\delta)\log\log(4n/\delta)\rceil$ we get that with probability at least $1-\delta/4:$
\begin{equation} \label{eq: sup2}
2\max_{j\in\{0,1,\dots,N-1\}}2^j K\mu_n(M_f(2^j K))
\leq 2\log(4n/\delta)\widetilde{\Lambda}^\beta_f(\mu)+\frac{2\log(4n/\delta)}{n}~.    
\end{equation}
In order to bound the last term in \eqref{eq: 3 sups}, we observe that the empirical measure satisfies for any $A\subset\Omega:\mu_n(A)<\frac{1}{n}\iff\mu_n(A)=0$, and that $M_f(s)\subset M_f(t)$ for $s>t$. Furthermore, by definition of $K=\widetilde{\Lambda}_f^\beta(\mu)$ we have $\mu(M_f(t))\leq\frac{K}{t}$, hence by Markov's inequality
\[
\Pr_S\left[\sup_{s\geq t} \mu_n(M_f(s))\neq0\right]
\leq \Pr_S\left[\mu_n(M_f(t))\neq0\right]
=\Pr_S\left[\mu_n(M_f(t))\geq\frac{1}{n}\right]
\leq \frac{nK}{t}~.
\]
For $t:=2^N K$ yields $\Pr_S\left[\sup_{s\geq 2^n K} \mu_n(M_f(s))\neq0\right]\leq \frac{n}{2^N}\leq\frac{\delta}{4}$. Combining this with \eqref{eq: sup1}, \eqref{eq: sup2} and plugging back into \eqref{eq: 3 sups}, we get that with probability at least $1-\delta/2:$
\[
\W_{X\sim\mu_n}[\Lambda^{\beta}_f(X)]
\leq
(1+2\log(4n/\delta))\widetilde{\Lambda}_f^\beta(\mu)+\frac{2\log(4n/\delta)}{n}
\leq
(1+2\log(4n/\delta))\overline{\Lambda}_f^\beta(\mu)+\frac{2\log(4n/\delta)}{n}~.
\]
Recalling \eqref{eq: hat_Lambda<log(n)W}, we get overall that
\[
\widehat{\Lambda}_f^\beta
\leq 2\log(n)\left[(1+2\log(4n/\delta))\overline{\Lambda}_f^\beta(\mu)+\frac{2\log(4n/\delta)}{n}\right]
~.
\]
Simplifying the expression above finishes the proof.

\end{proof}

\begin{proposition} \label{prop: approx extend g}
Under the same setting, for any $\gamma>0$ there exists an algorithm that given a sample $S\sim\Dcal^n$ and any function $\widehat{f}:S\to[0,1]$, provided that $n\geq N$ for $N=\widetilde{O}\left(\frac{\Ncal_{\Omega}(\gamma)+\log(1/\delta)}{\gamma}\right)$, constructs a function $f:\Omega\to[0,1]$ such that with probability at least $1-\delta/2:$
    \begin{itemize}
        \item $\norm{f-\widehat{f}}_{L_1(\mu_n)}\leq\gamma(1+2\widehat{\Lambda}_{\widehat{f}}^\beta)$. In particular $L_S(f)\leq L_S(\widehat{f})+\gamma(1+2\widehat{\Lambda}_{\widehat{f}}^\beta)$.
        \item $\overline{\Lambda}^\beta_f(\mu)\leq 5\widehat{\Lambda}_{\widehat{f}}^\beta$~.
    \end{itemize}
\end{proposition}

\begin{proof}
Throughout the proof, we denote for any point $x\in\Omega$, subset $B\subset \Omega$ and function $g:B\to[0,1]:$
\[
\Lambda^\beta_g(x,B):=\sup_{y\in B\setminus\{x\}}\frac{|g(x)-g(y)|}{\rho(x,y)^\beta}
~.
\]
Give the sample $S=(X_i,Y_i)_{i=1}^{n}$, we denote $S_x=(X_i)_{i=1}^{n}$.
Let $\gamma>0$. The algorithm constructs $f:\Omega\to[0,1]$ as follows:
    \begin{enumerate}
        \item Let $S_x(\gamma)\subset S_x$ consist of the $\lfloor\gamma n\rfloor$ points whose $\Lambda_{\widehat{f}}(\cdot,S_x)$ values are the largest (with ties broken arbitrarily), and $S'_x(\gamma):=S_x\setminus S_x(\gamma)$ be the rest.
        \item Let $A\subset S'_x(\gamma)$ be a $\gamma^{1/\beta}$-net of $S'_x(\gamma)$.
        \item Define $f:\Omega\to[0,1]$ to be the $\beta$-PMSE extension of ${\widehat{f}}$ from $A$ to $\Omega$ as defined in \defref{def: extension} (and analyzed throughout Appendix \ref{sec: extension}).
    \end{enumerate}
We will prove that $f$ satisfies both requirements.
For the first requirement, since $f|_A={\widehat{f}}|_A$ and $S_x= S'_x(\gamma)\uplus S_x(\gamma)$ we have
\[
\norm{f-{\widehat{f}}}_{L_1(\mu_n)}:=\frac{1}{n}\sum_{i=1}^{n}|f(x_i)-g(x_i)|
=\frac{1}{n}\sum_{x\in S_x(\gamma)\setminus A}|f(x)-{\widehat{f}}(x)|+\frac{1}{n}\sum_{x\in S'_x(\gamma)\setminus A}|f(x)-{\widehat{f}}(x)|~.
\]
The first summand above is bounded by $\gamma$ since $0\leq f,\widehat{f}\leq 1\implies|f(x)-{\widehat{f}}(x)|\leq 1$ and $|S_x(\gamma)|\leq \gamma n$. In order to bound the second term, we denote by $N_A:S'_x(\gamma)\to A$ to be the mapping of each element to its nearest neighbor in the net, and note that $\rho(x,N_A(x))\leq\gamma^{1/\beta}$. Then
\begin{align*}
\frac{1}{n}\sum_{x\in S'_x(\gamma)\setminus A}|f(x)-{\widehat{f}}(x)|
&\leq \frac{1}{n}\sum_{x\in S'_x(\gamma)\setminus A}\frac{\gamma}{\rho(x,N_A(x))^\beta}|f(x)-{\widehat{f}}(x)|
\\
&\leq \frac{\gamma}{n}\sum_{x\in S'_x(\gamma)\setminus A}\frac{|f(x)-{\widehat{f}}(N_A(x))|+|{\widehat{f}}(N_A(x))-{\widehat{f}}(x)|}{\rho(x,N_A(x))^\beta}
\\
&= \frac{\gamma}{n}\sum_{x\in S'_x(\gamma)\setminus A}\frac{|f(x)-f(N_A(x))|}{\rho(x,N_A(x))^\beta}+\frac{|{\widehat{f}}(N_A(x))-{\widehat{f}}(x)|}{\rho(x,N_A(x))^\beta}
\\
&\leq \frac{\gamma}{n}\sum_{x\in S'_x(\gamma)\setminus A}\Lambda^\beta_f(x,A)+\Lambda^\beta_{\widehat{f}}(x,A)
\\
\left[\mathrm{\thmref{thm: extension}}\right]& \leq \frac{2\gamma}{n}\sum_{x\in S'_x(\gamma)\setminus A}\Lambda^\beta_{\widehat{f}}(x,A)
\\
&\leq 2\gamma L
~.
\end{align*}
So overall we get $\norm{f-{\widehat{f}}}_{L_1(\mu_n)}\leq\gamma +2\gamma L=\gamma(1+2L)$ as claimed in the first bullet.

We move on to prove the second bullet. Let $U\subset \Omega$ be a $\frac{\gamma^{1/\beta}}{4}$-net of $\Omega$, $\Pi$ be its induced Voronoi partition and let $m:=|\Pi|\leq\Ncal_{\Omega}(\gamma^{1/\beta}/4)$. Let Consider the following partition of $\Pi$ into ``light'' and ``heavy'' cells:
\[
\Pi_l :=\left\{C\in\Pi\,:\,\mu_n(C)<n\gamma/m\right\},~~~\Pi_h:=\Pi\setminus\Pi_l~.
\]
We will now state three lemmas required for the proof, two of which are due to \citep{ashlagi2021functions}.

\begin{lemma} \label{lem: slope ratio}
    Suppose $A\subset\Omega$ and that $f:\Omega\to[0,1]$ is the $\beta$-PMSE extension of some function from $A$ to $\Omega$. Let $E\subset \Omega$ such that $\mathrm{diam}(E)^\beta\leq\frac{1}{2}\min_{x\neq x'\in A}\rho(x,x')^\beta$. Then $\sup_{x,x'\in E}\frac{\Lambda^\beta_f(x)}{\Lambda^\beta_f(x')}\leq 2$.
\end{lemma}

\begin{proof}
    Let $u_x^*,v_x^*\in A$ be the pair of points which satisfy $\Lambda^\beta_f(x)=\frac{f(v^*_x)-f(u^*_x)}{\rho(v^*_x,x)^\beta+\rho(u^*_x,x)^\beta}$. By assumption on $E$, we know that $2\mathrm{diam}(E)^\beta\leq\rho(v^*_x,u^*_x)^\beta \leq\rho(v^*_x,x)^\beta+\rho(u^*_x,x)^\beta$, hence $\rho(v^*_x,x)^\beta+\rho(u^*_x,x)^\beta+2\mathrm{diam}(E)^\beta\leq 2(\rho(v^*_x,x)^\beta+\rho(u^*_x,x)^\beta)$. We get
    \begin{align*}
    \Lambda_f^\beta(x')
    &\geq\frac{f(v^*_x)-f(u^*_x)}{\rho(v^*_x,x')^\beta+\rho(u^*_x,x')^\beta}
    \\
    &\geq \frac{f(v^*_x)-f(u^*_x)}{\rho(v^*_x,x)^\beta+\mathrm{diam}(E)^\beta+\rho(u^*_x,x)^\beta+\mathrm{diam}(E)^\beta}
    \\
    &\geq\frac{f(v^*_x)-f(u^*_x)}{2(\rho(v^*_x,x)^\beta+\rho(u^*_x,x)^\beta)}
    =\frac{1}{2}\Lambda^\beta_f(x)~.
    \end{align*}
\end{proof}

\begin{lemma}[\citealp{ashlagi2021functions}, Lemma 16] \label{lem: ash cell mass}
If $n\gamma^2\geq m$, then
\begin{align*}
    \Pr_{S\sim\Dcal^n}\left[\min_{C\in\Pi_h}\frac{\mu_n(C)}{\mu(C)}
    > \frac{1}{2}\right]&\geq 1-m\exp(-n\gamma/4m)~,
    \\
    \Pr_{S\sim\Dcal^n}\left[\max_{C\in\Pi_h}\frac{\mu_n(C)}{\mu(C)}
    < 2\right]&\geq 1-m\exp(-n\gamma/3m)~,
    \\
    \Pr_{S\sim\Dcal^n}\left[\sum_{C\in\Pi_l}\mu(C)<2\gamma\right]&\geq 1-\exp\left(-n(\gamma-\sqrt{m/n})^2/2\right)~.
\end{align*}
\end{lemma}

\begin{lemma}[\citealp{ashlagi2021functions}, Lemma 17] \label{lem: ash f norm}
    $\norm{f}_{\Hol^\beta}\leq \frac{2L}{\gamma}$.
\end{lemma}

Equipped with the three lemmas, we calculate
\begin{equation} \label{eq: adverserial avg slope bound}
\overline{\Lambda}^\beta_f(\mu)
=\int_{\Omega}\Lambda^\beta_f(x)d\mu
=\sum_{C\in\Pi_l}\int_{C}\Lambda^\beta_f(x)d\mu+\sum_{C\in\Pi_h}\int_{C}\Lambda^\beta_f(x)d\mu~.
\end{equation}
The first summand above is bounded with high probability using \lemref{lem: ash cell mass} and \lemref{lem: ash f norm}, since under the event described in \lemref{lem: ash cell mass} we have:
\begin{align*}
    \sum_{C\in\Pi_l}\int_{C}\Lambda^\beta_f(x)d\mu
    &\leq \sum_{C\in\Pi_l}\int_{C} \frac{2L}{\gamma}d\mu
    =\frac{2L}{\gamma}\sum_{C\in\Pi_l}\mu(C)
    \\
    &\leq \frac{2L}{\gamma} \cdot 2q
    =\frac{L}{4}~.
\end{align*}
In order to bound the second term in \eqref{eq: adverserial avg slope bound},
let $C\in\Pi,~x'\in C$ and note that by applying \lemref{lem: slope ratio} to $E:=S_x\cap C$ we get that $\Lambda^\beta_f(x')\leq 2\min_{x\in S_x\cap C}\Lambda^\beta_f(x)$. Thus, under the high probability event described in \lemref{lem: ash cell mass} we have
\begin{align*}
\sum_{C\in\Pi_h}\int_{C}\Lambda^\beta_f(x)d\mu
&\leq \sum_{C\in\Pi_h}\int_{C}2\min_{x\in S_x\cap C}\Lambda^\beta_f(x)d\mu
=2\sum_{C\in\Pi_h}\min_{x\in S_x\cap C}\Lambda^\beta_f(x)\mu(C)
\\
&\leq 4\sum_{C\in\Pi_h}\min_{x\in S_x\cap C}\Lambda^\beta_f(x)\mu_n(C)
= \frac{4}{n}\sum_{C\in\Pi_h}\sum_{x'\in S_x\cap C}\min_{x\in S_x\cap C}\Lambda^\beta_f(x)
\\
& \leq \frac{4}{n}\sum_{C\in\Pi_h}\sum_{x'\in S_x\cap C}\Lambda^\beta_f(x')
\leq \frac{4}{n}\sum_{x'\in S_x}\Lambda^\beta_f(x')
\leq 4L~,
\end{align*}
where the last inequality is due to the extension property of \thmref{thm: extension}. Overall, plugging these bounds into \eqref{eq: adverserial avg slope bound} and using the union bound to ensure all required events to hold simultaneously, we see that the desired second bullet holds holds with probability at least $1-m\exp(-n\gamma/4m)-\exp\left(-n(\gamma-\sqrt{m/n})^2/2\right)$. A straightforward computation shows that by our assumption on $n$ being large enough, this probability exceeds $1-\delta/2$ as required.

\end{proof}

We are now ready to finish the proof of \thmref{thm: alg regression}. Let $\gamma>0$.
By \corref{cor: approx erm}, we can construct $\widehat{f}:S\to[0,1]$ such that with probability at least $1-\delta/2:~L_S(\widehat{f})=0$ and $\widehat{\Lambda}^\beta_{\widehat{f}}\leq 5\log^2(4n/\delta)L$.
Assuming $n$ is appropriately large, we further apply \propref{prop: approx extend g} in order to obtain $f:\Omega\to[0,1]$ such that with probability at least $1-\delta/2:~f\in\overline{\Hol}^\beta_{25\log^2(4n/\delta)L}(\Omega)$ and also $L_S(f)\leq L_S(\widehat{f})+\gamma(1+2L)= \gamma(1+2L)$.
By the union bound, we get that with probability at least $1-\delta:$
\begin{align*}
L_\Dcal(f)
&=~1.01L_S(f) + (L_\Dcal(f)-1.01L_S(f))
\\
&\leq ~\gamma(1+2L)~+\sup_{f\in\overline{\Hol}^\beta_{25\log^2(4n/\delta)L}(\Omega)}(L_\Dcal(f)-1.01L_S(f))~.
\\
&\overset{(*)}{\leq}~\frac{\eps}{2}+\frac{\eps}{2}=\eps~,
\end{align*}
where $(*)$ is justified by setting $\gamma=\Theta(\eps/L)$ and applying \thmref{thm: uc} for appropriately large $n$.

\subsection{Proof of \thmref{thm: agnostic alg}} \label{sec: agnostic alg proof}

Given a sample $S=(X_i,Y_i)_{i=1}^{n}\sim\Dcal^n$, denote the empirically smooth class 
\[
\widehat{\Hol}:=\left\{f:\{X_1,\dots,X_{\lfloor n/2\rfloor}\}\to[0,1]\,:\,
\widehat{\Lambda}^\beta_f\leq 5\log^2(4n/\delta)L
\right\}~.
\]
Consider the following procedure:

\begin{enumerate}
    \item (\emph{Empirical cover}) Construct $h_1,\dots,h_N\in\widehat{\Hol}$ for maximal $N$ such that $\forall i\neq j\in[N]:~\norm{h_i-h_j}_{L_1(\mu_n)}\geq \frac{\epsilon}{4}$ .
    
    \item (\emph{Run realizable algorithm on cover}) For any $j\in[N]$, execute the realizable algorithm $\Acal_{\mathrm{realizable}}$ of \thmref{thm: alg regression} on the ``relabeled'' dataset $(X_i,h_j(X_i))_{i=1}^{\lfloor n/2\rfloor}$, and obtain $f_j:\Omega\to[0,1]$.
    \item (\emph{ERM}) Return $\arg\min_{f_1,\dots,f_N}\sum_{i=\lfloor n/2\rfloor+1}^{n}|f_j(X_i)-Y_i|$.
\end{enumerate}

We will now prove that the algorithm above satisfies the theorem. Let $f^*\in\arg\min_{f\in\overline{\Hol}_{L}^\beta(\Omega,\mu)}L_{\Dcal}(f)$,\footnote{We assume without loss of generality that the infimum is obtained. Otherwise we can take a function whose loss is arbitrarily close enough to the optimal value and continue with the proof verbatim.} 
and note that by \propref{prop: approx erm} (as explained in \corref{prop: approx erm}) we have $f^*\in\widehat{\Hol}$ with probability at least $1-\delta/2$. By construction, $h_1,\dots,h_N$ is a maximal $\frac{\epsilon}{4}$-packing of $\widehat{\Hol}$, which is known to imply that it is also a $\frac{\epsilon}{4}$-net \citep[Lemma 4.2.8]{vershynin2018high} with respect to the metric $L_1(\mu_n)$. In particular, this implies that there exists $j^*\in[N]$ such that
\[
\norm{f^*-h_{j^*}}_{L_1(\mu_n)}\leq\frac{\epsilon}{4}
\implies 
L_S(h_{j^*}) \leq L_S(f^*)+\frac{\epsilon}{4}~.
\]
Further note for any $j\in[N]:~h_j\in\widehat{\Hol}$,
so our realizable algorithm (as manifested in \propref{prop: approx extend g} for $\gamma=\Theta(\epsilon/L)$) when fed the ``smoothed'' labels $(X_i,h_j(X_i))_{i=1}^{\lfloor n/2\rfloor}$ will produce $f_j$ such that $L_S(f_j)\leq L_S(h_j)\leq \frac{\epsilon}{4}$ and $\overline{\Lambda}^\beta_{f_j}(\mu)\leq 5\widehat{\Lambda}_{h_j}^\beta
\leq 25\log^2(4n/\delta)L$. In particular
\[
L_S(f_{j^*})
\leq L_S(h_{j^*})+\frac{\epsilon}{4}
\leq L_S(f^*)+\frac{\epsilon}{2}~.
\]
Finally, by \eqref{eq: cov vs brack} and \thmref{thm: bracket} (which holds for any measure, in particular for the empirical measure $\mu_n$)
\begin{align*}
\log N
&\leq \log \Ncal_{\widehat{\Hol}}(\epsilon/2)
\\
&\leq \log \Ncal_{[\,]}(\widehat{\Hol},L_1(\mu_n),\epsilon)
\\
&\leq 
\log
\Ncal_{\Omega}\left(\left(\frac{\eps}{640\log^2(4n/\delta)L\log(1/\eps)}\right)^{1/\beta}\right)\cdot\log\left(\frac{16\log_2(1/\eps)}{\eps}\right)~.
\end{align*}
Hence, by a standard Chernoff-Hoeffding bound over the finite class $\{f_1,\dots,f_N\}$, step (3) of the algorithm yields $\frac{\epsilon}{2}$ excess risk as long as $\frac{n}{2}=\Omega\left( \frac{\log(N)+\log(1/\delta)
}{\epsilon^2}\right)$.

\subsection{Proof of \thmref{thm: lower}}\label{proof:lower}

We start by providing a simple structural result which we will use for our lower bound construction, showing that in any metric space there exists a sufficiently isolated point from a large enough subset.

\begin{lemma}
    There exists a point $x_0\in\Omega$ and a subset $K\subset\Omega$ such that
    \begin{itemize}
        \item
        $\forall x\in K:\rho(x_0,x)\geq\frac{\mathrm{diam}(\Omega)}{4}\,.$
        \item $\forall x\neq y\in K:\rho(x,y)\geq {(\eps/L)^{1/\beta}}\,.$
        \item $|K|=\left\lfloor\frac{\Ncal_{\Omega}((\eps/L)^{1/\beta})}{2}\right\rfloor\,.$
    \end{itemize}
\end{lemma}

\begin{proof}
    Denote $D:=\mathrm{diam}(\Omega)$, let $x_0,x_1$ be two points such that $\rho(x_0,x_1)> D/2$, and let $\Pi=\{C_0,C_1\}$ be a Voronoi partition of $\Omega$ induced by $\{x_0,x_1\}$. For $\gamma>0$, let $N_{\gamma}$ be a maximal $\gamma$-packing of $\Omega$. By the pigeonhole principle there must exist a cell $C_i\in\Pi$ such that $|C_i\cap N_{\gamma}|\geq |N_{\gamma}|/2$, which we assume without loss of generality to be $C_1$. Now note that any $x\in C_1$ satisfies $\rho(x,x_0)\geq\frac{1}{2}\rho(x,x_0)+\frac{1}{2}\rho(x,x_1)\geq \frac{1}{2}\rho(x_0,x_1)>D/4$. Finally, set $\gamma:=\eps^{1/\beta}$ and let $K\subset C_1\cap N_{\gamma}$ be any subset of size $\left\lfloor\frac{\Ncal_{\Omega}((\eps/L)^{1/\beta})}{2}\right\rfloor$.
\end{proof}

Given $x_0,K$ from the lemma above, we denote $\overline{K}=\{x_0\}\cup K$ and define the distribution $\mu$ over $\Omega$ supported on $\overline{K}$ such that $\mu(x_0)=1-\frac{\eps}{2}$ and $\mu(x)=\frac{\eps}{2|K|}$ for all $x\in K$.
From now on, the proof is similar to a classic lower bound strategy for VC classes in the realizable case (e.g. \citealp[Proof of Theorem 3.5]{kearns1994introduction}). To that end, it is enough to provide a distribution over functions in $\overline{\Hol}_L^\beta(\Omega,\mu)$ such that with 
constant probability any algorithm must suffer significant loss for some function supported by the distribution.

We define such a distribution over functions $\overline{f}:\overline{K}\to\{0,1\}$ as follows: $\Pr[\overline{f}(x_0)=0]=1$, while for any $x\in K:\Pr[\overline{f}(x)=0]=\Pr[\overline{f}(x)=1]=\frac{1}{2}$ independently of other points. We will now show that any such $\overline{f}:\overline{K}\to\{0,1\}$ is average H\"older smooth with respect to $\mu$. Indeed, for every $x\in K:$
\[
\Lambda_{\overline{f}}^\beta(x)
=\sup_{x'\in \overline{K}\setminus\{x\}}\frac{|\overline{f}(x)-\overline{f}(x')|}{\rho(x,x')^\beta}
\leq \frac{1}{\eps/L}=\frac{L}{\eps}~,
\]
while
\[
\Lambda_{\overline{f}}^\beta(x_0)
=\sup_{x'\in \overline{K}\setminus\{x_0\}}\frac{|\overline{f}(x_0)-\overline{f}(x')|}{\rho(x_0,x')^\beta}
\leq \frac{1}{\mathrm{diam}(\Omega)/4}=\frac{4}{\mathrm{diam}(\Omega)}~,
\]
hence
\[
\overline{\Lambda}_{\overline{f}}^\beta(x)
=\mu(x_0)\Lambda_{\overline{f}}^\beta(x_0)+\sum_{x\in K}\mu(x)\Lambda_{\overline{f}}^\beta(x)
\leq \frac{4}{D}+\frac{L}{2}\leq L~.
\]

Finally, we define the (random) function $f^*:\Omega\to[0,1]$ to be the $\beta$-PMSE extension of $\overline{f}$ from $\overline{K}$ to $\Omega$ as defined in \defref{def: extension}, and note that $f^*$ satisfies the required smoothness assumption. Setting $\Dcal$ over $\Omega\times[0,1]$ to have marginal $\mu$ and $Y=f^*(X)$, we ensure that $\Dcal$ is indeed realizable by $\overline{\Hol}_L^\beta(\Omega)$.

Now assume $A$ is a learning algorithm which is given a sample $S$ of size $|S|\leq \frac{|K|}{4\eps}$ and produces $A(S):\Omega\to[0,1]$.
We call a point $x\in K$ "misclassified" by the algorithm if $|A(S)(x)-f^*(x)|\geq\frac{1}{2}$, and denote the set of misclassified points by $M\subset K$.
Recalling that $\forall x\in K: \Pr[f(x)=0]=\Pr[f(x)=1]=\frac{1}{2}$ independently, and that $\mu(x)=\frac{\eps}{2|K|}$, we observe that with probability at least $\frac{1}{2}$ the algorithm will misclassify more than $|K|/2$ points.\footnote{Indeed, denoting $C=K\setminus M$ we see that
$\Pr[|C|\geq |K|/8]\leq \frac{8}{|K|}\cdot\E[|C|]=\frac{8}{|K|}\cdot\frac{|S|}{2}\cdot\mu(K)\leq\frac{8}{|K|}\cdot\frac{|K|}{8\eps}\cdot\frac{\eps}{2}
=\frac{1}{2}$.
} Thus, we get that with probability at least $\frac{1}{2}:$
\[
L_{\Dcal}(A(S))
=\E_{X\sim\mu}[|A(S)(X)-f^*(X)|]
\geq \sum_{x\in M}\mu(x)\cdot|A(S)(x)-f^*(x)|
\geq \frac{|K|}{2}\cdot\frac{\eps}{2|K|}\cdot\frac{1}{2}=\frac{\eps}{8}~.
\]
By rescaling $\eps$, we see that in order to obtain $L_\Dcal(A(S))\leq\eps$ the sample size must be of size 
\[
\Omega\left(\frac{|K|}{\eps}\right)=\Omega\left(\frac{\Ncal_\Omega((\eps/L)^{1/\beta})}{\eps}\right)~.
\]

\subsection{Proofs from \secref{sec: examples}} \label{proof:examples}

\paragraph{Proof of Claim~\ref{example 1}.}

Let $\beta\in(0,1)$. Consider the unit segment $\Omega=[0,1]$ with the standard metric, equipped with the probability measure $\mu$ with density $\frac{d\mu}{dx}=\frac{1}{Z}|x-\half|^{\frac{\beta-1}{2}}$ (where $Z=\int_{0}^{1}|x-\half|^{\frac{\beta-1}{2}}<\infty$ is a normalizing constant). We examine the function $f(x)=\boldsymbol{1}[x>\half]$ which is clearly not H\"older continuous since it is discontinuous. Furthermore,
\begin{gather*}
\mu(\{x:\Lambda^1_f(x)\geq t\})
=\mu\left(\left\{\left|x-\half\right|\leq \frac{1}{t}\right\}\right)
=\frac{2}{Z}\int_{0}^{1/t}x^{\frac{\beta-1}{2}}dx
\asymp t^{-\frac{\beta+1}{2}}
\\
\implies \widetilde{\Lambda}^1_f
=\sup_{t>0} t\cdot\mu(\{x:\Lambda^1_f(x)\geq t\})
\asymp \sup_{t>0}t^{\frac{1-\beta}{2}}=\infty
~,
\end{gather*}
hence $f\notin\widetilde{\Lip}_{M}(\Omega,\mu)$ for all $M>0$.
On the other hand, $\Lambda^\beta_f(x)=\frac{1}{|x-\half|^{\beta}}$ so
\[
\overline{\Lambda}^\beta_f
=\int_{0}^{1}\Lambda^{\beta}_f(x) d\mu
=\frac{1}{Z}\int_{0}^{1}\frac{|x-\half|^{\frac{\beta-1}{2}}}{|x-\half|^{\beta}}dx
=\frac{1}{Z}\int_{0}^{1}\frac{1}{|x-\half|^{\frac{\beta+1}{2}}}dx
\overset{(\beta<1)}{<}\infty~,
\]
thus $f\in\overline{\Hol}^\beta_L(\Omega)$ for some $L<\infty$. Note that by normalizing the function, the claim holds even for $L=1$.

\paragraph{Proof of Claim~\ref{example 2}.}

    Let $\beta\in(0,1)$. Consider the unit segment $\Omega=[0,1]$ with the standard metric, equipped with the probability measure $\mu$ with density $\frac{d\mu}{dx}=\frac{1}{Z}|x-\half|^{\beta-1}$ (where $Z=\int_{0}^{1}|x-\half|^{\beta-1}dx<\infty$ is a normalizing constant). We examine the function $f(x)=\boldsymbol{1}[x>\half]$. Note that for any $x\neq\half:\,\Lambda^1_f(x)=\frac{1}{|x-\half|}$, hence
\[
\mu(\{x:\Lambda^1_f(x)\geq t\})=\mu\left(\left\{x:|x-\half|\leq\frac{1}{t}\right\}\right)
=\frac{2}{Z}\int_{0}^{1/t}{x^{\beta-1}}dx
\asymp t^{-\beta}
~.
\]
This shows that
\[
\widetilde{\Lambda}^1_f
=\sup_{t>0}t\cdot\mu(\{x:\Lambda^1_f(x)\geq t\})\asymp\sup_{t>0}t^{1-\beta}=\infty
~,
\]
hence $f\notin\widetilde{\Lip}_{M}(\Omega,\mu)$ for all $M>0$.
Furthermore, for $x\neq\half:\,\Lambda_f^\beta(x)=\frac{1}{|x-\half|^\beta}$ so
\[
\overline{\Lambda}^\beta_f
=\int_{0}^{1}\frac{1}{|x-\half|^\beta}d\mu
=\frac{1}{Z}\int_{0}^{1}\frac{1}{|x-\half|}dx
=\infty~,
\]
hence $f\notin\widetilde{\Hol}_{M}^\beta(\Omega,\mu)$ for all $M>0$.
On the other hand
\begin{gather*}
\mu(\{x:\Lambda^\beta_f(x)\geq t\})=\mu(\{|x-\half|\leq t^{-1/\beta}\})
=\frac{2}{Z}\int_{0}^{t^{-1/\beta}}x^{\beta-1}dx\asymp t^{-1}~
\\
\implies \widetilde{\Lambda}^\beta_f
=\sup_{t>0}t\cdot\mu(\{x:\Lambda^\beta_f(x)\geq t\})
<\infty~,
\end{gather*}
thus $f\in\widetilde{\Hol}^\beta_L(\Omega)$ for some $L<\infty$. Note that by normalizing the function, the claim holds even for $L=1$.

\section{Discussion}

In this work, we have defined a notion of an average-H\"older smoothness, extending 
the
average-Lipschitz one
introduced by \citet{ashlagi2021functions}.
Using proof techniques based on bracketing numbers, 
we have 
established the minimax rate for average-smoothness classes in the realizable setting with respect to the $L_1$ risk up to logarithmic factors, and have provided a nontrivial learning algorithm that attains this nearly-optimal learning rate. Moreover, we have also provided yet another learning algorithm for the agnostic setting. All of these results improve upon previously known rates even in the special case of average-Lipschitz classes.

A few notes are in order. First, the choice of focusing on $L_1$ risk as opposed to general $L_p$ losses is merely a matter of conciseness, as to avoid introducing additional parameters.
Indeed, the only place throughout the proofs which we use the $L_1$ loss is in the proof of \propref{prop: bracket to uc}, where we show that the 
loss-class
$\Lcal_{\Fcal}:=\{x\mapsto|f(x)-f^*(x)|:f\in\Fcal\}$ satisfies
\[
\Nbr(\Lcal_{\Fcal},L_1(\mu),\alpha)\leq\Nbr(\Fcal,L_1(\mu),\alpha)
~.
\]
It is easy to show via essentially the same proof that for any $p\in[1,\infty)$, the $L_p$-composed loss-class satisfies $\Nbr(\Lcal_{\Fcal},L_1(\mu),\alpha)\leq\Nbr(\Fcal,L_1(\mu),\alpha^{1/p})$, and the remaining proofs can be invoked verbatim.
This yields a realizable sample complexity (in the typical, $d$-dimensional case) of order 
$N=\widetilde{O}\left(\frac{L^{d/p\beta}}{\eps^{(d+p\beta)/p\beta}}\right)$,
or equivalently $L_p$-risk decay rate of 
$L_{\Dcal}(f)= \widetilde{O}\left(\frac{L^{d/(d+p\beta)}}{n^{p\beta/(d+p\beta)}}\right)$ which are also easily translatable to their corresponding agnostic rates.

Focusing again on $L_1$ minimax rates of average-H\"older classes, it is interesting to compare them to the minimax rates of ``classic'' (i.e., worst-case) H\"older classes. \citet{schreuder2020bounding} has shown the minimax risk to be of order $n^{-\beta/d}$, whereas we showed the average-smooth case has the slightly worse rate of $n^{-\beta/(d+\beta)}$ (which cannot be improved, due to our matching lower bound). 
However, comparing the rates alone is rather misleading, since both risks are multiplied by a factor depending on their corresponding H\"older {constant}, which can be considerably smaller in the average-case result.
Still, it is interesting to note that in the asymptotic regime there is a marginal advantage in case the learned function is worst-case H\"older, as opposed to H\"older on average.

Our work leaves open several questions. 
A relatively straightforward one is to compute the minimax rates and construct an optimal algorithm for the \emph{classification} setting, which is not addressed by this paper.
Moreover, there is a slight mismatch between our established upper and lower bounds in the agnostic setting, ranging between $\widetilde{O}(n^{-\beta/(d+2\beta)})$ and ${\Omega}(n^{-\beta/(d+\beta)})$. Closing this gap is an interesting problem which we leave for future work.

\paragraph{Acknowledgments.}

AK is partially supported by the Israel Science Foundation (grant No. 1602/19), an Amazon Research Award, and the Ben-Gurion University Data Science Research Center.
The authors would like to thank Max Hopkins for insightful discussions regarding the agnostic reduction, and for pointing out a bug (as well as the fix) in an earlier version of the paper.

\bibliographystyle{plainnat}
\bibliography{bib}

\appendix

\section{Minimal $\beta$-slope H\"older extension} \label{sec: extension}

In this section we describe a procedure that extends H\"older functions in an optimally smoothest manner at every point, as it will serve as a crucial ingredient in our proofs. That is, given a subset of a metric space $A\subset\Omega$ and a function $f:\Omega\to[0,1]$, it produces $F_A:\Omega\to[0,1]$ such that
\begin{enumerate}
    \item It extends $f|_A:~F_A|_A=f|_A$.
    \item For any $\widetilde{F}:\Omega\to[0,1]$ that extends $f|_A$, it holds that $\Lambda^\beta_{F_A}(x)\leq\Lambda^\beta_{\widetilde{F}}(x)$ for all $x\in\Omega$.
\end{enumerate}

Such a procedure was described for Lipschitz extensions (namely when $\beta=1$) in \citet{ashlagi2021functions}. The purpose of this section is to generalize this procedure to any H\"older exponent.

Throughout this section we fix $\beta\in(0,1],~\emptyset\neq A\subset\Omega$ and $f:\Omega\to[0,1]$, and will always assume the following.

\begin{assumption} \label{ass: extension}
$\norm{f|_A}_{\Hol^\beta}<\infty$ and $\mathrm{diam}(A)<\infty$.
\end{assumption}
Keeping in mind that the case we are really interested in is when $A$ is finite (i.e. a sample), the conditions above are trivially satisfied. Nonetheless, everything we will present 
continues to hold in
this more general setting.
For $u,v\in A$ 
we introduce the following notation:
\begin{align*}
    R_x(u,v)&:=\frac{f(v)-f(u)}{\rho(x,v)^\beta+\rho(x,u)^\beta}~,
    \\
    F_x(u,v)&:=f(u)+R_x(u,v)\rho(x,u)^\beta~,
    \\
    R^*_x&:=\sup_{u,v\in A}R_x(u,v)~,
    \\
    W_x(\eps)&:=\{(u,v)\in A\times A:R_x(u,v)>R_x^*-\eps\}~,~~~~0<\eps<R^*_x
    \\
    \Phi_x(\eps)&:=\{F_x(u,v):(u,v)\in W_x(\eps)\}~.
\end{align*}

\begin{definition} \label{def: extension}
We define the $\beta$-pointwise minimal slope extension ($\beta$-PMSE) to be the function $F_A:\Omega\to\reals$ satisfying
\[
F_A(x):=\lim_{\eps\to0^+}\Phi_x(\eps)
~.
\]
In the degenerate case in which $f(u)=f(v)$ 
for all $u,v\in A$,
define $F_A(x):=f(u)$ for some (and hence any) $u\in A$.
\end{definition}

\begin{theorem} \label{thm: extension}
Let $\emptyset\neq A\subset\Omega,~f:\Omega\to[0,1]$, such that \assref{ass: extension} holds.
Then $F_A:\Omega\to[0,1]$ is well defined, and satisfies for any $x\in\Omega:~\Lambda^\beta_{F_A}(x)\leq \Lambda^\beta_f(x)$.
Furthermore, if $A$ is finite, then $F_A(x)$ can be computed for any $x\in\Omega$ within $O(|A|^2)$ arithmetic operations.
\end{theorem}

\begin{remark}
When $R_x(\cdot,\cdot)$ has a unique maximizer $(u^*_x,v^*_x)\in A\times A$, the definition of $F_A$ simplifies to
\begin{equation} \label{eq: u*,v*}
F_A(x)=f(u^*_x)+\frac{\rho(x,u^*_x)^\beta}{\rho(x,u^*_x)^\beta+\rho(x,v^*_x)^\beta}(f(v^*_x)-f(u^*_x))
~.
\end{equation}
We conclude that under \assref{ass: extension}, we can assume without loss of generality that for each $x\in\Omega$ there is such a unique maximizer (since the function is well defined, thus does not depend on the choice of the maximizer). Furthermore, this readily shows that when $A$ is finite, we can compute $F_A(x)$ for any $x\in\Omega$ within $O(|A|^2)$ arithmetic operations -- simply by finding this maximizer.
\end{remark}

\begin{proof}(of \thmref{thm: extension})

We will assume that there exist $u,v\in A$ such that $f(u)\neq f(v)$, since the degenerate (constant extension) case is trivial to verify. This assumption implies that $R_x^*>0$. It is also easy to verify that $\sup_{x\in\Omega}R_x^*<\infty\iff\|f\|_{\Hol^\beta}<\infty$.

\begin{lemma}
$F_A$ is well defined. Namely, under \assref{ass: extension} the limit $\lim_{\eps\to0^+}\Phi_x(\eps)\in[0,1]$ exists.
\end{lemma}

\begin{proof}
Fix $x\in\Omega$ (we will omit the $x$ subscripts from now on). Let $\eps<R^*/2$, $(u,v),(u',v')\in W(\eps)$. Note that $R(u,v)>0$ and that $F(u,v)=f(v)-R(u,v)\rho(x,v)^\beta$. Hence
\begin{equation} \label{eq: f(u)<=f(v)}
    f(u)\leq F(u,v)\leq f(v)
    ~,
\end{equation}
and the same clearly holds if we replace $(u,v)$ by $(u',v')$. Assume without loss of generality that $F(u,v)\leq F(u',v')$, hence $f(u)\leq F(u,v)\leq F(u',v')\leq f(v')$. We get
\begin{align*}
    R(u',v')+\eps
    &>R^*
    \\
    &\geq \frac{f(v')-f(u)}{\rho(x,v')^\beta+\rho(x,u)^\beta}
    \\&=\frac{f(v')-F(u',v')+F(u,v)-f(u)}{\rho(x,v')^\beta+\rho(x,u)^\beta}
    +\frac{F(u',v')-F(u,v)}{\rho(x,v')^\beta+\rho(x,u)^\beta}
    \\
    &=\frac{R(u',v')\rho(x,v')^\beta+R(u,v)\rho(x,u)^\beta}{\rho(x,v')^\beta+\rho(x,u)^\beta}
    +\frac{F(u',v')-F(u,v)}{\rho(x,v')^\beta+\rho(x,u)^\beta}
    \\
    &\geq\frac{R(u',v')\rho(x,v')^\beta+(R(u',v')-\eps)\rho(x,u)^\beta}{\rho(x,v')^\beta+\rho(x,u)^\beta}
    +\frac{F(u',v')-F(u,v)}{2\mathrm{diam}(A)^\beta}
    \\
    &\geq R(u',v')-\eps+\frac{F(u',v')-F(u,v)}{2\mathrm{diam}(A)^\beta}
\end{align*}
\[
\implies |F_x(u,v)-F_x(u',v')|\leq 4\eps\,\mathrm{diam}(A)^\beta~.
\]
We conclude that if $\mathrm{diam}(A)<\infty$ then $\lim_{\eps\to0^+}\Phi_x(\eps)$ indeed exists.

\end{proof}

It remains to prove the optimality of the $\beta$-slope.
Throughout the proof we will denote for any $u\neq v\in\Omega:$
\[
S(u,v):=\frac{|F_A(u)-F_A(v)|}{\rho(u,v)^\beta}
~,
\]
and for any point $x\in\Omega$, subset $B\subset \Omega$ and function $g:\Omega\to[0,1]$ we let
\[
\Lambda^\beta_g(x,B):=\sup_{y\in B\setminus\{x\}}\frac{|g(x)-g(y)|}{\rho(x,y)^\beta}
~.
\]

The proof is split into three claims.
\paragraph{Claim I.} $\forall x\in\Omega\setminus A:~\Lambda^\beta_{F_A}(x,A)\leq \Lambda^\beta_f(x,A)$.

Let $x\in\Omega\setminus A$, and let $(u^*,v^*)\in A\times A$ be its associated maximizer of $R_x$. Recall \eqref{eq: f(u)<=f(v)} from which we can deduce that $F_A(u^*)\leq F_A(x)\leq F_A(v^*)$.
Also note that a simple rearrangement based on \eqref{eq: u*,v*} (and the fact that $f$ and $F_A$ agree on $A$) shows that
$S(u^*,x)=R_x(u^*,v^*)=S(x,v^*)$.
Furthermore, we claim that $\Lambda_{F_A}^{\beta}(x,A):=\sup_{y\in A\setminus\{x\}}S(x,y) =S(x,u^*)$. If this were not true then we would have $S(x,y)>S(x,u^*)=S(x,v^*)$ for some $y\in A\setminus\{x,u^*,v^*\}$. Using the mediant inequality, if $f(y)\geq f(x)$ this implies
\[
R_x(u^*,y)
=\frac{f(y)-f(u^*)}{\rho(x,y)^\beta+\rho(x,u^*)^\beta}
=\frac{F_A(y)-F_A(x)+F_A(x)-F_A(u^*)}{\rho(x,y)^\beta+\rho(x,u^*)^\beta}
>S(x,u^*)=R_x(u^*,v^*)
~,
\]
while if $f(y)<f(x)$ then
\[
R_x(y,v^*)
=\frac{f(v^*)-f(y)}{\rho(x,v^*)^\beta+\rho(x,y)^\beta}
=\frac{F_A(v^*)-F_A(x)+F_A(x)-F_A(y)}{\rho(x,v^*)^\beta+\rho(x,y)^\beta}
>S(x,v^*)=R_x(u^*,v^*)
~,
\]
both contradicting the maximizing property of $(u^*,v^*)$ - so indeed $\Lambda_{F_A}^{\beta}(x,A)=S(x,u^*)=S(x,v^*)$.
In particular, we see that if $F_A(x)\geq f(x)$ then
\[
\Lambda^\beta_f(x,A)
=\sup_{y\in A\setminus\{x\}}\frac{|f(y)-f(x)|}{\rho(y,x)^\beta}
\geq\frac{f(v^*)-f(x)}{\rho(v^*,x)^\beta}
\geq\frac{F_A(v^*)-F_A(x)}{\rho(v^*,x)^\beta}
=S(x,v^*)
=\Lambda_{F_A}^{\beta}(x,A)
~,
\]
while if $F_A(x)<f(x)$ then
\[
\Lambda^\beta_f(x,A)
=\sup_{y\in A\setminus\{x\}}\frac{|f(x)-f(u)|}{\rho(x,y)^\beta}
\geq\frac{f(x)-f(u^*)}{\rho(x,u^*)^\beta}
>\frac{F_A(x)-F_A(u^*)}{\rho(x,u^*)^\beta}
=S(x,u^*)
=\Lambda_{F_A}^{\beta}(x,A)
~,
\]
proving Claim I in either case.

\paragraph{Claim II.} $\forall x\in\Omega\setminus A:~\Lambda^{\beta}_{F_A}(x,\Omega\setminus A)\leq\Lambda^{\beta}_{F_A}(x,A)$, in particular $\Lambda^{\beta}_{F_A}(x,\Omega)=\Lambda^{\beta}_{F_A}(x,A)$.

It suffices to show that for any $x,y\in\Omega\setminus A:$
\[
S(x,y)\leq \min\{\Lambda_{F_A}^\beta(x,A),\Lambda_{F_A}^\beta(y,A)\}
~,
\]
since taking the supremum of the left hand side over $y\in\Omega\setminus A$ shows the claim.
Let $(u^*_x,v^*_x),(u^*_y,v^*_y)$ the associated maximizers of $R_x,R_y$ respectively, and note that by definition we have
\begin{equation} \label{eq: Lambda x,A inequality towards}
\Lambda^\beta_{F_A}(x,A)=\sup_{z\in A\setminus\{x\}}S(x,z)\geq \max\{S(x,u^*_y),S(x,v^*_y)\}
~.
\end{equation}
We assume without loss of generality that $\Lambda_{F_A}^\beta(x,A)\leq\Lambda_{F_A}^\beta(y,A)$, and recall that by \eqref{eq: f(u)<=f(v)} we can deduce that $F_A(u^*_x)\leq F_A(x)\leq F_A(v^*_x)$ and $F_A(u^*_y)\leq F_A(y)\leq F_A(v^*_y)$. Now suppose by contradiction that $S(x,y)>\Lambda_{F_A}^\beta(x,A)$. If $F_A(x)\leq F_A(y)$ then
\begin{align*}
F_A(v^*_y)
&=F_A(x)+\rho(x,y)^\beta S(x,y)+\rho(y,v^*_y)^\beta\Lambda_{F_A}^\beta(y,A)
\\
&>F_A(x)+\rho(x,y)^\beta \Lambda_{F_A}^\beta(x,A)+\rho(y,v^*_y)^\beta\Lambda_{F_A}^\beta(x,A)
\\
&\geq
F_A(x)+\rho(x,v^*_y)^\beta \Lambda_{F_A}^\beta(x,A)~,
\end{align*}
thus $S(x,v^*_y)=\frac{|F_A(x)-F_A(v^*_y)|}{\rho(x,v^*_y)^\beta}>\Lambda^\beta_{F_A}(x,A)$ which contradicts \eqref{eq: Lambda x,A inequality towards}. On the other hand, if $F_A(x)>F_A(y)$ then
\begin{align*}
F_A(x)
&=F_A(u^*_y)+\rho(u^*_y,y)^\beta \Lambda_{F_A}^\beta(y,A)
+\rho(y,x)^\beta S(x,y)
\\
&>F_A(u^*_y)+\rho(u^*_y,y)^\beta \Lambda_{F_A}^\beta(x,A)
+\rho(y,x)^\beta \Lambda_{F_A}^\beta(x,A)
\\
&\geq F_A(u^*_y)+\rho(u^*_y,x)^\beta \Lambda_{F_A}^\beta(x,A)
~,
\end{align*}
thus $S(x,u^*_y)=\frac{|F_A(x)-F_A(u^*_y)|}{\rho(x,u^*_y)^\beta}>\Lambda^\beta_{F_A}(x,A)$ which contradicts \eqref{eq: Lambda x,A inequality towards}, and proves claim Claim II.

\paragraph{Claim III.} $\forall x\in A:~ \Lambda^{\beta}_{F_A}(x,\Omega)
=\Lambda^{\beta}_{F_A}(x,A)
\leq\Lambda^{\beta}_{f}(x,\Omega)$.

Let $x\in A$. Assume towards contradiction that there exists $y\notin A$ such that
\[
\Lambda_{F_A}(x,\Omega)\geq
S(x,y)
>\Lambda^\beta_{F_A}(x,A)
~.
\]
We denote by $(u^*_y,v^*_y)\in A\times A$ the maximizer of $R_y(\cdot,\cdot)$. Recall that since $x\in A$, in the proof of Claim I we showed that $S(x,y)\leq S(y,u^*_y)=S(y,v^*_y)$. If $F_A(x)\leq F_A(y)\leq F_A(v^*_y)$ then
\begin{align*}
S(x,v^*_y)
&=\frac{F_A(v^*_y)-F_A(x)}{\rho(v^*_y,x)^\beta}
\geq
\frac{F_A(v^*_y)-F_A(y)+F_A(y)-F_A(x)}{\rho(v^*_y,y)^\beta+\rho(x,y)^\beta}
\\
&\geq \min\{S(y,v^*_y),S(x,y)\}
=S(x,y)
>\Lambda_{F_A}^\beta(x,A)
~,
\end{align*}
while on the other hand if $F_A(x)> F_A(y)\geq F_A(u^*_y)$ then
\begin{align*}
S(x,u^*_y)
&=\frac{F_A(x)-F_A(u^*_y)}{\rho(x,u^*_y)^\beta}
\geq\frac{F_A(x)-F_A(y)+F_A(y)-F_A(u^*_y)}{\rho(x,y)^\beta+\rho(u^*_y,y)^\beta}
\\
&\geq \min\{S(x,y),S(y,u^*_y)\}
=S(x,y)
>\Lambda_{F_A}^\beta(x,A)
~,
\end{align*}
where in both calculations we used the mediant inequality.
Both inequalities above contradict the definition of $\Lambda_{F_A}^\beta(x,A)$, thus proving Claim III.

\paragraph{Combining the ingredients.}

We are now ready to finish the proof.
For $x\in\Omega$, if $x\in A$ then Claim III provides the desired inequality. Otherwise, if $x\in\Omega\setminus A$ then
\[
\Lambda_{F_A}^\beta(x,\Omega)
\overset{\mathrm{Claim\,II}}{=}\Lambda_{F_A}^\beta(x,A)
\overset{\mathrm{Claim\,I}}{\leq}\Lambda^\beta_f(x,A)
\leq \Lambda^\beta_f(x,\Omega)
~.
\]

\end{proof}

\end{document}